\documentclass[11pt]{article} 
\usepackage[utf8]{inputenc}
\RequirePackage{amsthm,amsmath,amsfonts,amssymb,mathtools}
\RequirePackage{graphicx}
\RequirePackage[square,numbers]{natbib} 

\usepackage{authblk} 
\usepackage[dvipsnames]{xcolor}
\usepackage[plainpages=false,pdfpagelabels,
colorlinks=true,linkcolor=RedOrange,
urlcolor=MidnightBlue,citecolor=MidnightBlue]{hyperref}
\providecommand{\keywords}[1]{\small \textbf{\textit{Keywords---}} #1}
\usepackage[font={small,it}]{caption}
\usepackage{fullpage}
\usepackage{tikz}
\usetikzlibrary{positioning}
\usepackage{algorithm, algpseudocode}
\usepackage{subcaption}

\usepackage{my-macros}

\begin{document}
\title{Learning When the Concept Shifts: Confounding, Invariance, and Dimension Reduction}
\author{Kulunu Dharmakeerthi}
\author{YoonHaeng Hur}
\author{Tengyuan Liang}

\affil{The University of Chicago}

\maketitle
\begin{abstract}
Practitioners often face the challenge of deploying prediction models in new environments with shifted distributions of covariates and responses. With observational data, such shifts are often driven by unobserved confounding, and can in fact alter the concept of which model is best. This paper studies distribution shifts in the domain adaptation problem with unobserved confounding. We postulate a linear structural causal model to account for endogeneity and unobserved confounding, and we leverage exogenous invariant covariate representations to cure concept shifts and improve target prediction. We propose a data-driven representation learning method that optimizes for a lower-dimensional linear subspace and a prediction model confined to that subspace. This method operates on a non-convex objective---that interpolates between predictability and stability---constrained to the Stiefel manifold, using an analog of projected gradient descent. We analyze the optimization landscape and prove that, provided sufficient regularization, nearly all local optima align with an invariant linear subspace resilient to distribution shifts. This method achieves a nearly ideal gap between target and source risk. We validate the method and theory with real-world data sets to illustrate the tradeoffs between predictability and stability. 
\end{abstract}




\section{Introduction}
Practitioners often deploy predictive models in new environments where the covariate and response distributions have shifted. Consider two observational data sets collected from different environments: a source environment used for training and a target environment where the model is deployed. With observational data sets, a model maximizing prediction accuracy in the source environment may experience a substantial drop in accuracy in the target environment, often due to unobserved confounding factors lurking in the environments. Such confounding can fundamentally alter the data distribution across environments and, in turn, alter the underlying concept of which statistical model is best.

A predictive model that generalizes well to new data sets is the statistical ideal, yet theory and methodology weaken when the environment underpinning the data set shifts. Recently, a growing literature in causal inference has shown the utility of comparing data sets from distinct environments to identify causal relations \citep{peters2016causal, HeinzeDeml2017InvariantCP, Pfister2019InvariantCP, Arjovsky2019InvariantRM}. At the same time, a body of works in machine learning (ML) and artificial intelligence (AI) has been tackling prediction under distribution shifts, leading to a rich field called domain adaptation \citep{ben2006analysis, ben2010theory, blitzer-etal-2006-domain, blitzer2007learning, tzeng2014deep, long2015learning, ganin2016domain}. While both lines of research aim to address the problem of distribution shifts, they utilize two distinct notions of stability under different problem settings, which we discuss below.

\paragraph{Causal Stability} Consider a covariate-response pair $(X, Y)$ and its joint probability distribution $\cP_{\cE}(X, Y)$ that differs across environments $\cE$. We consider two environments $\cE \in \{\cS, \cT\}$ corresponding to source and target, respectively. Due to the presence of unobserved confounding factors, both the \textit{covariate distribution} $\cP_{\cE}(X)$, the marginal distribution of $X$, and the \textit{conditional concept} $\cP_{\cE}(Y| X)$, the conditional distribution of $Y$ given $X$, could be shifted and altered by the change in the environment. The principle of independent mechanisms in causality \citep{peters2017elements} points to a plausible source of invariance: the mechanism producing the effect given the cause should remain unchanged across environments. This principle can be made practical. For instance, the invariant risk minimization work \citep{peters2016causal, HeinzeDeml2017InvariantCP, Pfister2019InvariantCP, rothenhausler2021anchor} seeks to identify a causally ``stable'' representation, that is, a map $X \mapsto \phi(X)$ such that $\cP(Y | \phi(X))$, the conditional concept given $\phi(X)$, stays invariant across environments (hence no subscript $\cE$). This line of work requires labeled data across environments, specifically, it assumes access to samples from both $\cP_\cS(X, Y)$ and $\cP_\cT(X, Y)$.

\paragraph{Distributional Stability} Unlike the setting above, domain adaptation considers the case where practitioners observe labeled covariate-response pairs from the source distribution $\cP_{\cS}(X, Y)$, but only \textit{unlabeled} covariate samples from a target distribution $\cP_{\cT}(X)$. The seminal work of \citet{ben2006analysis} derived an upper bound on target risk, a key to domain adaptation. Given a representation $\phi: X \mapsto \phi(X)$, the upper bound involves a balancing act between (i) the source risk of the model under the representation $\phi$, and (ii) a distance between probability distributions $\cP_{\cS}(\phi(X))$ and $\cP_{\cT}(\phi(X))$. Here, a notion of distributional ``stability'' arises. One prefers a representation $X \mapsto \phi(X)$ where the difference between $\cP_{\cS}(\phi(X))$ and $\cP_{\cT}(\phi(X))$ is small, holding source risk fixed. This notion of invariance is also natural: if the distribution shifts significantly, the model selected based on source data may suffer severely when extrapolating to target data.

\paragraph{This Paper} Studying concept shifts induced by unobserved confounding---a central topic in causal inference with observational data---has received comparatively less attention in domain adaptation. We propose a structural causal model to address the challenge of domain adaptation in the presence of distribution shifts caused by unobserved confounding factors. This structural model unites notions of stability and the methods used to enforce them, developed independently in the ML/AI and statistics literatures.

Much like the role of instrumental variables in the presence of confounding, we find that leveraging an exogenous and invariant covariate representation can address both concept and covariate shifts lurking in the environments. Further, in the context of our structural model, notions of causal stability and distributional stability will align. With this understanding, we will motivate, from first principles, a stability-regularized risk minimization method that aims for invariance and stability when adapting to new observational domains.

\subsection{A Structural Causal Model}
Throughout the paper, we postulate a Structural Causal Model (SCM) for unobserved confounding inspired by the idea of instrumental variables \citep{StockTrebbi2003JEP, Pearl2009Book}. The crucial difference is that, in our model, these exogenous variables---resembling instruments---are latent and unobserved. 

\begin{definition}[Model]
	\label{def:SCM-model}
	A pair of covariate and response, denoted as $(X, Y) \in \mathbb{R}^d \times \mathbb{R}$, are generated from mutually independent exogenous random variables $E \in \R^r, Z \in \R^k$ with $r, k \le d$ and $U \in \R, W \in \R^d$ via the following structural causal model,
	\begin{align}
	    Y & = \langle \beta^\star, X \rangle + \langle \gamma, E \rangle + U \label{eq:model_Y} \;, \\
	    X & = \Theta Z + \Delta E + W \label{eq:model_X} \;,
	\end{align}
	where $\beta^\star \in \R^d$, $\gamma \in \R^r$, $\Theta \in \R^{d \times k}$, and $\Delta \in \R^{d \times r}$ are unknown parameters; see Figure \ref{fig:model_diagram}. 

	The \textit{unobserved confounding} variable $E$ models the environmental influence, whose distribution shifts across settings. Concretely, in domain adaptation, the distribution of $E$ shifts from the source distribution $\rho_\cS$ to the target distribution $\rho_\cT$, where $\rho_\cS, \rho_\cT$ are probability distributions defined on $\R^r$. 
	The \textit{latent invariant} $Z$ models a source of randomness that is not subject to environment shifts: a stable, exogenous structure shared across both source and target environments. It plays a similar role to an instrumental variable, with the notable difference that it is unobserved across environments.
	The \textit{exogenous noise} $U, W$ are mutually independent and mean-zero, and their distributions do not depend on the environment.
\end{definition}

\begin{figure}[!htb]
    \centering
    \begin{tikzpicture}[
        roundnode/.style={circle, draw=black!60, fill=black!5, ultra thick, minimum size=15mm}, scale=0.8, transform shape]
        \node[roundnode]      (X)                              {$X$};
        \node[dashed, roundnode]        (E)       [above right =2cm and 1.2cm of X] {$E$};
        \node[dashed, roundnode]      (Z)       [left=3cm of X] {$Z$};
        \node[roundnode]        (Y)       [right=3cm of X] {$Y$};
        \draw[dotted, very thick, ->] (Z.east) -- (X.west);
        \draw[very thick, ->] (X.east) -- (Y.west);
        \draw[very thick, ->] (1.6, 2.75) -- (0.2, 0.75);
        \draw[dotted, very thick, ->] (3, 2.75) -- (4.4, 0.75);
    \end{tikzpicture}
    \caption{Diagram visualizing the model in Definition \ref{def:SCM-model}. The endogenous confounding variable $E$ lurking in the environment and the exogenous invariant variable $Z$ are both latent and unobserved.}
    \label{fig:model_diagram}
\end{figure}

\vspace*{-2em}

\paragraph{Observational Data and the Adaptation Problem}
Consider the source and target environments $\cS$ and $\cT$. The joint distributions of $(X, Y)$ under $E \sim \rho_\cS$ and $E \sim \rho_\cT$ are denoted as $\cP_\cS(X, Y)$ and $\cP_\cT(X, Y)$, with $\cP_{\cS}(X)$ and $\cP_{\cT}(X)$ denoting the covariate distributions. We use $\E_\cE$ to denote the expectation w.r.t.\ the joint distribution under the environment $\cE \in \{\cS, \cT\}$. For the expectation of random variables depending only on $Z, U, W$, we drop the subscript $\cE$ as it is invariant across environments.

The practitioner observes labeled covariate-response pairs jointly drawn from the source distribution $\cP_{\cS}(X, Y)$, but only unlabeled covariate data from a target distribution $\cP_{\cT}(X)$. Given the observational data, the practitioner aims to estimate a linear model based on $\cP_{\cS}(X, Y)$ and $\cP_{\cT}(X)$ that performs well in the target environment.

\paragraph{Confounding and Concept Shift}
A few remarks follow for our model. First, suppose the dotted line connections in Figure \ref{fig:model_diagram} are removed. In that case, we arrive at the prototypical well-specified covariate shift setting: environment shifts only affect the covariate distribution, not the conditional distribution of $Y$ given $X$. In this setting, without confounding, using (weighted) least squares to estimate a model from $X$ to $Y$ consistently is possible. Our model incorporating the dotted lines---capturing unobserved confounding---naturally extends the covariate shift setting. 
Second, since the unobserved confounding variable $E$ lurking in the environment influences both $X$ and $Y$, the least squares estimate is biased and unreliable when deploying to new environments. As we shall see in Proposition \ref{prop:conceptshift}, the distribution shift in $E$ under this structural model induces a \textit{concept shift}, meaning that the best linear model depends on the environment. Finally, introducing the invariant, exogenous $Z$ may remind readers of instrumental variables. However, in our domain adaptation problem, $Z$ is unobserved; we merely postulate its existence. We develop a methodology that leverages invariance across environments to learn a stable, lower-dimensional linear subspace without directly observing $Z$.

\subsection{A Motivating Example}
\label{sec:motivating-example}
We peek at data from an influential economic study by \citet{tabellini2010}. Tabellini uses this data to explore the impact of historic social, economic and political development on the European economy in the late 90s. We will treat this data as observational, and consider the covariate-response pair:
\begin{equation*}
    X = \{\textit{Institutions}, \textit{Literacy}, \textit{Enrollment}, \textit{Culture}, \textit{Urbanization}\}, \, Y = \{\textit{Economic Output}\} \;.
\end{equation*}

Due to circumstances, suppose we only have access to data from rural regions (the source environment) and wish to construct a linear predictor for $Y$ in historically urban areas (the target environment). How do we find an accurate predictor for \textit{Economic Output} in the target environment? Considering our feature set, $X$, we would certainly expect distribution shifts in at least one coordinate---`Urbanization'. Furthermore, Tabellini establishes `Institutions' as an instrumental variable that is independent of `Urbanization', and so, we would also expect at least one feature to be stable amid this environment shift. We can capture this setting in our data model; `Institutions' and `Urbanization' are natural candidates as components in $Z$ and $E$, respectively. However, we do not assume this structure a priori, and hope to leverage the stability of certain features, like `Institutions', in a data-driven way. We will revisit this example in Section~\ref{sec:example}, after introducing the stability-regularized risk minimization method.

\paragraph{Notations} 
Let $\|v\|$ denote the standard Euclidean norm of a vector $v \in \R^d$. For any symmetric matrix $M \in \R^{d \times d}$, let $\lambda_{\max}(M), \lambda_{\min}(M)$ denote the largest, smallest eigenvalues of $M$, respectively; if $M$ is positive semidefinite, let $M^{1/2}$ denote its matrix root and $\|v\|_M := \| M^{1/2} v \|$ for $v \in \R^d$. Also, $\|\cdot\|_{\mathrm{op}}, \|\cdot\|_{\mathrm{F}}$ are matrix norms (spectral, Frobenius, respectively). The Stiefel manifold is defined as $\mathrm{St}(d, \ell) := \{ A \in \R^{d \times \ell}: \ A^\top A = I_\ell\}$. Finally, $R_\cE$ denotes the squared risk depending on environment $\cE \in \{\cS, \cT\}$, namely, $R_\cE(\beta) = \E_\cE[(Y - \langle X, \beta \rangle)^2]$, where the expectation is w.r.t.\ the joint distribution $\cP_{\cE}(X, Y)$.

\subsection{Organization and Contribution}

The paper studies domain adaptation in the presence of unobserved confounding. Unobserved confounding is a likely issue for a practitioner attempting to extrapolate with observational data. As seen in Figure \ref{fig:model_diagram}, we extend the standard covariate shift setting to incorporate confounding factors lurking in the environment. These latent variables simultaneously influence covariate $X$ and response $Y$. 

In Section~\ref{sec:why-risk-min}, we explore the shortcomings of vanilla risk minimization when both covariate distribution and ``concept'' are expected to shift under the linear SCM. As shown in Proposition \ref{prop:DREI}, an invariant subspace in this model will unify concept stability and distributional stability, two distinct ideas in the literature a priori. Proposition \ref{prop:TRI} further establishes the necessity and benefit of leveraging a lower-dimensional, invariant subspace for predictive power. 

Guided by insights into the invariant subspace, we solve the domain adaptation task in Section~\ref{sec:methodology}. Given distributional access to the unlabeled target, $\cP_\cT(X)$, and labeled data from the source, $\cP_\cS(X,Y)$, we seek a methodology gauging towards the inaccessible target risk by learning a subspace that balances predictability and stability/invariance. Concretely, we seek an estimator composed of a linear subspace representation, $V \in \R^{d \times \ell}$, and a low-dimensional ridge regressor, $\alpha \in \R^\ell$, jointly optimized according to the objective
{\small
\begin{equation}
	\label{eq:obj0}
    V, \ \alpha := \argmin_{V, \alpha} \E_\cS[(Y- \langle X, V \alpha \rangle)^2] + \upsilon \|\alpha \|^2 + \tfrac{\eta}{2} \underbrace{\|V^\top ( \E_\cT[XX^\top] - \E_\cS[XX^\top] ) V\|_{\mathrm{F}}^2}_{(\ast)} \;.
\end{equation}
}
We define the composed estimator $\beta^{\upsilon, \eta} =  V \alpha$.

Proposition \ref{prop:unifUB} derives \eqref{eq:obj0} as an upper bound on the target risk. Therefore, the $(\ast)$ term is a natural notion of invariance for domain adaptation under the linear SCM. Importantly,  \eqref{eq:obj0} is an actionable proxy for optimization based on the information available, and in Section~\ref{sec:stiefel-opt}, a practical first-order manifold optimization method is devised to navigate its non-convex landscape. By optimizing over the choice of regularization parameters $\eta, \upsilon$ in \eqref{eq:obj0}, the target risk can effectively be optimized. This empirical fact is demonstrated using real-world data examples in Section~\ref{sec:example}. 

Moving to Section~\ref{sec:theory}, we provide a theoretical characterization of the landscape of the non-convex manifold optimization. When the regularization parameter $\eta$ is sufficiently large, we show that almost all local optima align well with an exogenous, invariant linear subspace and are resilient to distribution shifts driven by endogenous confounding factors. Indeed, denoting $\mathrm{Endo}\subset \R^{d}$ as the endogenous subspace corresponding to $\Delta$ in \eqref{eq:model_X}, we find the first-order stationary point $V \in \mathrm{St}(d, \ell)$ of \eqref{eq:obj0} satisfies:
\begin{equation*}
    \max_{\bv \in V, \bw \in \mathrm{Endo}} \big(\cos \angle (\bv, \bw) \big)^6 \leq O\left(\frac{1}{\upsilon \eta^2}\right) \;.
\end{equation*}
This alignment result leads to a stability bound between the target and source risks that directly addresses the domain adaptation problem. For any stationary point of \eqref{eq:obj0}, 
\begin{equation*}
    R_\cT(\beta^{\upsilon, \eta}) - R_\cS(\beta^{\upsilon, \eta}) \leq \inf_{\beta \perp \mathrm{Endo}} \underbrace{\{R_\cT(\beta) - R_\cS(\beta)\}}_{\text{oracle term}} ~+~ O\left(\frac{1}{\upsilon^{4/3}\eta^{2/3}}\right) \;.
\end{equation*}
Namely, a predictive model using the learned lower-dimensional subspace could incur a nearly ideal gap between target and source risk, quantified by the oracle term. The theoretical results on invariance alignment and risk stability corroborate empirical observations in real-world data sets. Despite the difficult nature of non-convex manifold optimization, the practical first-order method often identifies good local optima for domain adaptation. Lastly, we also provide a finite sample error analysis to quantify the gap between the optimization and its plug-in estimation.

We provide a detailed literature review and discussion, deferred to Appendix A in the Supplementary Material, due to the space limit. All technical proofs are collected in Appendix C for the same reason.

\section{Undesired Properties of Source Risk Minimization}
\label{sec:why-risk-min}
In this section, we study the linear SCM in Definition~\ref{def:SCM-model}, identify the undesired theoretical properties of vanilla source risk minimization, and further characterize the benefits of leveraging an invariant subspace to enforce stable learning and improve the target risk.

\subsection{Concept Shift, Confounding, and Subspace Invariance}
We first state the assumptions used in this section.
\begin{assumption}
	\label{asmp:base}
	In Definition~\ref{def:SCM-model}, the exogenous variables $Z, U, W$ are mean-zero, namely, $\E[Z] = 0, \E[U] = 0, \E[W] = 0$, and satisfy $\E[W W^\top] = \tau^2 I_d$ with $\tau > 0$.
\end{assumption}

For each environment $\cE \in \{\cS, \cT\}$, classical risk minimization takes the following form:
\begin{equation}
	\label{eq:risk_minimization}
    \argmin_{\beta \in \R^d} R_\cE(\beta) = \argmin_{\beta \in \R^d} \E_\cE[(Y - \langle X, \beta \rangle)^2] \;.
\end{equation}
In the SCM in Definition~\ref{def:SCM-model}, when the endogeneity parameter $\gamma \neq 0$, the unobserved variable $E$ lurking in the environment plays a confounding role in the relationship between $X$ and $Y$. A natural question is: will there be a linear model simultaneously optimal for source and target risk minimization? The answer is no, thus implying that the best model will be a moving target. We shall see that a shift in the distribution of the confounding variable $E \sim \rho_{\cS}$ to $E \sim \rho_{\cT}$ can lead to a stark difference in the best linear model across distributions. The following proposition formally defines the above phenomenon as a \textit{concept shift}.

\begin{proposition}[Concept Shift]
\label{prop:conceptshift}
	Under Assumption \ref{asmp:base}, the risk minimization \eqref{eq:risk_minimization} admits a unique minimizer, denoted as $\beta_\cE$ for $\cE \in \{\cS, \cT\}$. Suppose $[\Theta, \Delta] \in \mathrm{St}(d, k + r)$. Then, if $\E_\cS[E E^\top] \neq \E_\cT[E E^\top]$, there exists $\gamma \in \R^r$ such that $\beta_\cS \neq \beta_\cT$, namely, the best linear predictor (concept) shifts across the two environments for some endogeneity parameter $\gamma$.
\end{proposition}

\begin{remark}
	It is immediate to show that under Assumption~\ref{asmp:base}, the best linear model is well-defined and given as $\beta_\cE = \beta^\star + (\E_\cE[X X^\top])^{-1}  \Delta \E_\cE[E E^\top] \gamma$. The above result shows that the best linear model shifts as long as the environment changes. The curious reader may wonder whether the Bayes optimal model $\E_{\cE}[Y|X]$ also changes. If we assume in addition that $(E, Z, W)$ is drawn from a multivariate Gaussian and $E$ is mean-zero, we can show $\E_{\cE}[Y|X] = \langle  \beta^\star + (\E_\cE[X X^\top])^{-1}  \Delta \E_\cE[E E^\top] \gamma, X \rangle$, which matches the best linear model. Therefore, the Bayes optimal concept is also environment dependent. The simple derivation above also shows that the re-weighting method of \citet{Shimodaira2000ImprovingPI}, based on the likelihood ratio of the covariate distributions, will not fix the concept shift issue. Due to endogeneity, a new methodology is needed.
\end{remark}

In plain language, source risk minimization fails to recover a model that stays invariant across environments. The concept shift induced by the movement of the second moments of $E$ reflects a bias that cannot be reconciled through traditional least squares. The root cause of the above concept shift is endogeneity. Inspired by the instrumental variable literature, we instead resort to exogenous linear subspaces of covariates invariant to environment shifts. The following proposition shows two desiderata---dimension reduction and invariance---can be achieved simultaneously.

\begin{proposition}[Subspace Invariance]
	\label{prop:DREI}
	Under Assumption \ref{asmp:base}, let $\beta^\Theta_\cE$ be the best linear predictor restricted to the linear subspace $\Theta \in \mathrm{St}(d, k)$ for each environment $\cE \in \{\cS, \cT\}$: 
	\begin{equation*}
		\beta^\Theta_\cE := \Theta \alpha^\Theta_\cE, ~ \text{where} ~ \alpha^\Theta_\cE = \argmin_{\alpha \in \R^k} \E_\cE[(Y - \langle  X, \Theta \alpha \rangle)^2] \;.
	\end{equation*}
	Suppose $[\Theta, \Delta] \in \mathrm{St}(d, k + r)$. Then, for any $\gamma \in \R^r$, we have $\beta^\Theta_\cS = \beta^\Theta_\cT$.
\end{proposition}

\begin{remark}
	The above result shows that deconfounding is possible with an invariant subspace. In our model, the effect of the confounding variable, $E$, is restricted to the linear subspace $\Delta$. Therefore, the covariate distribution projected to a lower-dimensional linear subspace, $\Theta$, remains stable despite the distribution shifts from $\cP_{\cS}(X, Y)$ to $\cP_{\cT}(X, Y)$. Proposition~\ref{prop:DREI} should be read in contrast to Proposition~\ref{prop:conceptshift}; risk minimization constrained to this ``invariant'' linear subspace resists arbitrary confounding effects parametrized by $\gamma$. And so, restricting to the invariant, exogenous space $\Theta$ yields stability in learned concepts, resilient to environment shifts.

	In view of the decomposition, $Y = \langle \Theta^\top \beta^\star , Z \rangle + \langle \Delta^\top \beta^\star + \gamma , E \rangle + \textrm{noise}$, the subspace model can be shown as $\beta^\Theta_{\mathcal{E}} \equiv \Theta \Theta^\top \beta^\star$. This decomposition separates variation in $Y$ into invariant, exogenous component $\langle \Theta^\top \beta^\star , Z \rangle$ and environment dependent, endogeneous component $\langle \Delta^\top \beta^\star + \gamma , E \rangle$. In particular, when $(I - \Theta \Theta^\top)\beta^\star = 0$, learning in the exogenous space consistently recovers the true signal; $\beta^\Theta_{\mathcal{E}} = \beta^\star$.
\end{remark}

Proposition~\ref{prop:DREI} only demonstrates the existence of an invariant linear subspace $\Theta$. However, since the exogenous variable $Z$ is unobserved, a data-driven approach to learning such a lower-dimensional linear subspace is still largely unclear. Motivated by invariance and stability, in Section~\ref{sec:methodology}, we will develop a methodology that inherits a manifold optimization problem to learn a lower-dimensional linear subspace. Later in Section~\ref{sec:theory}, we will derive a theoretical characterization of the subspace alignment between the invariant $\Theta$ and a first-order stationary point of the manifold optimization. A target risk bound exploiting this alignment will also be established therein.

\subsection{Target Risk Improvement via Invariance}

For domain adaptation, the lingering question is whether leveraging the invariant subspace $\Theta$ can improve the target risk compared to $\beta_{\cS}$, the vanilla source risk minimizer. This section provides a sufficient and necessary condition for the above question. The following proposition delineates when one can simultaneously achieve three desiderata: target risk improvement, invariance, and dimension reduction.

\begin{proposition}[Target Risk Improvement]
	\label{prop:TRI}
	Denote $\Lambda_\cE := \E_\cE[E E^\top]$ for $\cE \in \{\cS, \cT\}$. Under Assumption \ref{asmp:base}, and the assumption that $[\Theta, \Delta] \in \mathrm{St}(d, k + r)$ with $k + r = d$. Then, the linear subspace predictor indexed by $\Theta$ based on source risk, that is, $\beta^\Theta_\cS$, has a smaller target risk than the usual source risk minimizer $\beta_\cS$, namely, $R_\cT(\beta^\Theta_\cS) < R_\cT(\beta_\cS)$ if and only if 
	\begin{equation*}
		\|\Delta^\top \beta^\star + (\tau^2 I_r + \Lambda_\cT)^{-1} \Lambda_\cT \gamma\|_{\tau^2 I_r  + \Lambda_\cT}^2 < \|(\tau^2 I_r + \Lambda_\cS)^{-1} \Lambda_\cS \gamma - (\tau^2 I_r + \Lambda_\cT)^{-1}\Lambda_\cT \gamma\|_{\tau^2 I_r  + \Lambda_\cT}^2 \;.
	\end{equation*}
\end{proposition}

The right-hand side can be interpreted as the amount of concept shift, and the left-hand side can be interpreted as the sub-optimality of target risk due to the (invariant) subspace model. A special case helps to unpack the result. Consider $(I - \Theta \Theta^\top)\beta^\star = 0$. Then, the above expression reads
\begin{equation*}
	\|  \Lambda_\cT \gamma \|_{ (\tau^2 I_r  + \Lambda_\cT)^{-1} } < \| \tau^2 \underbrace{(\Lambda_\cT - \Lambda_\cS)}_{\text{environment shift}} (\tau^2 I_r + \Lambda_\cS)^{-1} \gamma  \|_{ (\tau^2 I_r  + \Lambda_\cT)^{-1} } \;.
\end{equation*}
Conceptually, when the environment shift---parametrized by the second moment of the confounding variable $E$---is large enough, the subspace model strictly improves upon the vanilla source risk minimizer. When such a condition holds, target risk improvement, invariance, and dimension reduction can be achieved simultaneously. 

Before concluding this section, we use the following simple example to elucidate the condition in Proposition~\ref{prop:TRI}.
\begin{example}
	Consider the simple setting where $\Lambda_\cS = \sigma^2_s \cdot I_r, \ \Lambda_\cT = \sigma^2_t  \cdot I_r$, and $\Delta^\top \beta^\star = c  \cdot \gamma$ with a scalar $c \in \R$. Proposition \ref{prop:TRI} boils down to: $R_\cT(\beta^\Theta_\cS) < R_\cT(\beta_\cS)$ if and only if
	\begin{align*}
		\left(c + \frac{\sigma^2_t}{\sigma^2_t + \tau^2}\right)^2 < \left(\frac{\sigma^2_s}{\sigma^2_s + \tau^2} - \frac{\sigma^2_t}{\sigma^2_t + \tau^2}\right)^2.
	\end{align*}
	Fixing other parameters $\sigma_t, \sigma_s, \tau$, this reduces to inequalities for the scalar parameter $c$;
	(i) \textit{Target Rich Regime} with $\sigma_t > \sigma_s$:
	$
			\frac{-\sigma_s^2}{\sigma_s^2 + \tau^2} + \frac{-2\tau^2(\sigma_t^2 - 			\sigma_s^2)}{(\sigma_t^2 + \tau^2)(\sigma_s^2 + \tau^2)} < c <  			\frac{-\sigma_s^2}{\sigma_s^2 + \tau^2}  \;;
	$
	(ii) \textit{Source Rich Regime} with $\sigma_s > \sigma_t$:
	$
			\frac{-\sigma_s^2}{\sigma_s^2 + \tau^2} < c <  			\frac{-\sigma_s^2}{\sigma_s^2 + \tau^2} + \frac{-2\tau^2(\sigma_t^2 - 			\sigma_s^2)}{(\sigma_t^2 + \tau^2)(\sigma_s^2 + \tau^2)} \;. 
	$
	In summary, the magnitude of confounding and the amount of environment shift altogether determine when invariance renders risk improvement.
\end{example}

So far, we have shown that leveraging subspace invariance allows one to surpass the limitations of source risk minimization in situations where the concept and the covariate shift. In the next section, we will propose a new domain adaptation method that learns a linear subspace resilient to shifting environments, making the insights obtained in this section actionable. The method trades off predictability and stability, using a lower-dimensional representation as a vehicle.

\section{Methodology: Domain Adaptation via Manifold Optimization}
\label{sec:methodology}
The previous section briefly discusses the fundamental tradeoff between invariance and predictive performance. In general, an invariant subspace can improve upon the source risk minimizer but may not be optimal for the target risk in domain adaptation. This section contributes to delineating this tradeoff in the SCM defined in Definition \ref{def:SCM-model}. Concretely, we design a representation learning method for domain adaptation, parametrized by subspace projections, and navigate the tradeoff between representation invariance and risk minimization. Though the proposed manifold optimization is non-convex, we demonstrate that a standard first-order method works both empirically in Section~\ref{sec:example} with real data sets, and later provably in Section~\ref{sec:theory}, where we develop theory matching the empirics. 

We motivate our main methodology in two ways: (i) as a surrogate for an upper bound on the target risk directly, and (ii) as a tradeoff to balance robustness/invariance and predictive performance.

\subsection{A Surrogate for the Target Risk}
In typical domain adaptation, labels from the target environment $Y \sim \cP_\cT(Y)$ are unavailable, and thus $R_\cT(\beta)$ cannot be directly minimized. To circumvent this issue, we design a simple surrogate objective that does not require labels from the target environment, requiring only information about $\cP_{\cS}(X, Y)$ and $\cP_{\cT}(X)$. In the sequel we will denote $\Sigma_\cE := \E_\cE[X X^\top]$ for $\cE \in \{\cS, \cT\}$. We start with a simple algebraic relation between the source and the target.
\begin{proposition}
	\label{prop:equality}
	Under Assumption \ref{asmp:base} and $\Delta^\top \Delta = I_r$, we have
	\begin{equation}
		\label{eq:R_difference_equality}
		R_\cT(\beta) = R_\cS(\beta) +  \big\langle \beta - (\beta^\star + \Delta \gamma) , (\Sigma_\cT - \Sigma_\cS) (\beta - (\beta^\star + \Delta \gamma) ) \big\rangle \quad \forall \beta \in \R^d \;.
	\end{equation}
\end{proposition}

The geometry of the target risk landscape involves a balance of two terms with clear distance interpretations. The first term $R_\cS(\beta)$ is the source risk. The second term encodes a distance between $\beta$ and the endogenous component, $\beta^\star + \Delta \gamma$, under a geometry determined by the covariance shift matrix $\Sigma_\cT - \Sigma_\cS$. The second term is not actionable based on data because $\beta^\star + \Delta \gamma$ is unknown. In the following proposition, we define an actionable objective, which serves as a surrogate upper bound for the target risk, for all $\beta$ and $\beta^\star + \Delta \gamma$.

To make the presentation simple, we impose the following additional assumption. 
\begin{assumption}[Target Richer than Source]
	\label{asmp:richer-target}
	Let $\Lambda_\cE := \E_\cE[E E^\top]$ for $\cE \in \{\cS, \cT\}$. Assume that $D := \Sigma_{\cT} - \Sigma_{\cS} \succ 0$.
\end{assumption}
Note that the assumption is for simplicity of presentation: we can state more general results by separating the positive and negative parts of $\Lambda_{\cT} - \Lambda_{\cS}$.
More importantly, we use Assumption \ref{asmp:richer-target} to focus on the interesting regime of distribution shift where the target distribution strictly dominates the source. This is when out-of-domain extrapolation can happen. Informally speaking, the Loewner order, $\Lambda_{\cT} \succ \Lambda_{\cS}$, indicates more variability present in the target domain.

With Assumption \ref{asmp:richer-target} in hand, an upper bound on target risk can be derived, which suggests a practical proxy for $R_\cT(\beta)$. To elucidate the low-dimensional subspace dependence, we explicitly construct the estimator $\beta$ to be a composite map of a lower-dimensional representation, parametrized by $V \in \R^{d \times \ell}$, and a linear model restricted to the representation, parametrized by $\alpha \in \R^\ell$.
\begin{proposition}[Surrogate for Target]
\label{prop:unifUB}
	If Assumptions \ref{asmp:base} and \ref{asmp:richer-target} hold with $\Delta^\top \Delta = I_r$, then for any $(V, \alpha) \in \R^{d \times \ell} \times \R^\ell$ and $\xi, \zeta >0$, we have $R_\cS(V\alpha) \leq R_{\cT}(V\alpha)$ and
	\begin{equation}
		\label{eq:unifUB}
		R_{\cT}(V\alpha)  \leq R_\cS(V\alpha) + \tfrac{(1+\xi)\zeta}{2} \|\alpha\|^4 +  \tfrac{(1+\xi)}{2\zeta} \| V^\top D V \|_{\mathrm{F}}^2 + (1+\tfrac{1}{\xi}) \langle \beta^\star + \Delta \gamma , D (\beta^\star + \Delta \gamma) \rangle \;.
	\end{equation}
\end{proposition}

Proposition~\ref{prop:unifUB} highlights that the gap between target and source risks of a composite estimator $\beta = V \alpha$ can be controlled by two complexities: $\|\alpha\|$ and $\|V^\top(\Sigma_{\cT} - \Sigma_{\cS})V\|_{\mathrm{F}}$. The term $\|V^\top(\Sigma_{\cT} - \Sigma_{\cS})V\|_{\mathrm{F}}$ quantifies the alignment of the subspace $V$ with the directions of covariance shift; in essence, a term describing stability for extrapolation. The term may also be interpreted as a quadratic maximum mean discrepancy (MMD) between covariate distributions in the linear subspace $V$. Indeed, under the map $\psi_V(X) = V^\top X X^\top V$,
\begin{align}
\label{eqn:MMD}
    \left\| V^\top (\E_\cT[XX^\top]  - \E_\cS[XX^\top] ) V \right\|_{\mathrm{F}} = \sup_{M: \| M\|_{\mathrm{F}} \leq 1} ~ \E_\cT[\langle M, \psi_V(X) \rangle] - \E_\cS[\langle M, \psi_V(X) \rangle] \;.
\end{align}

Finally, note that the right-hand side of \eqref{eq:unifUB} is a regularized source risk minimization, regularized by ridge and stability penalties. More importantly, it is an actionable surrogate objective for domain adaptation, solely depending on $\cP_{\cS}(X, Y)$ and $\cP_{\cT}(X)$.

\subsection{An Optimization Procedure on the Stiefel Manifold}
\label{sec:stiefel-opt}

Motivated by the upper bound on the target risk shown in Proposition~\ref{prop:unifUB}, we define the following penalized objective to search for a linear subspace $V$, and a model restricted to that subspace parametrized by $\alpha$,
\begin{equation}
	\label{eqn:objective}
	F_{\upsilon, \eta}(V, \alpha) := \tfrac{1}{2} \left\{ R_\cS(V \alpha) + \upsilon \|\alpha \|^2 + \tfrac{\eta}{2} \|V^\top ( \Sigma_\cT - \Sigma_\cS ) V\|_{\mathrm{F}}^2 \right\} \;. 
\end{equation}
By optimizing over a linear subspace $V \in \R^{d \times \ell}$ that balances predictive power and stability, the above objective provides an explicit tradeoff. Note that the two-level optimization is convex in $\alpha$ for fixed $V$, where the inner optimum $\min_{\alpha \in \mathbb{R}^{\ell}} F_{\upsilon,\eta}(V,\alpha)$ is attained at 
\begin{equation}
	\label{eqn:inner-ridge}
	\alpha_{V} := (V^\top \E_{\cS}[X X^\top] V + \upsilon I_{\ell} )^{-1} V^\top \E_{\cS}[XY] \;.
\end{equation}
We search for a representation confined to the Stiefel manifold $V \in \mathrm{St}(d, \ell)$. Putting things together, we arrive at the following optimization on the Stiefel manifold,
\begin{equation}
	\label{eqn:opt-stiefel}
	\min_{V \in \mathrm{St}(d, \ell)} \Phi_{\upsilon, \eta}(V) :=  \min_{V \in \mathrm{St}(d, \ell)} F_{\upsilon, \eta}(V,\alpha_{V})
\end{equation}
where $\alpha_V$ is defined in \eqref{eqn:inner-ridge}. The problem \eqref{eqn:opt-stiefel} is non-convex in $V$; however, we will derive provable results for any first-order stationary point of the landscape. Informally speaking, as we shall see in Section~\ref{sec:theory}, under certain conditions, almost all local optima demonstrate good alignment with the invariant subspace. 

Let us first introduce a skeletal implementation of a first-order manifold optimization method to find local optima of $\Phi_{\upsilon, \eta}(V)$. Then, we move to fill in the essential background on the geometry of the Stiefel manifold to rationalize the algorithm.
\begin{algorithm}
\caption{Optimization on Stiefel Manifold}\label{alg:armijo}
\begin{algorithmic}
    \Require Initial point: $V_0 \in \mathrm{St}(d, \ell)$.
    \For{$k = 0, 1, 2, \ldots$}
        \State Gradient: $G_k \gets (I - V_k V_k^\top) ((\Sigma_\cS V_k \alpha_{V_k} - \E_\cS[XY]) \alpha_{V_k}^\top  + \eta DV_k V_k^\top D V_k )$ as per \eqref{eqn:riem-grad}.
        \State Retraction: $V_{k+1}(t) := (V_k + t \cdot G_k) (I_\ell + t^2 \cdot G_k^\top G_k )^{-1/2}$ for $t \in \mathbb{R}_+$ as per \eqref{eqn:retraction}.
		\State Line-Search: choose step-size $t = t_k$ via line-search, and set $V_{k+1} \gets V_{k+1}(t_k)$.
    \EndFor
\end{algorithmic}
\end{algorithm}

\paragraph{Geometry on the Stiefel Manifold}
A few essential geometric items must be defined to understand the first-order method optimizing \eqref{eqn:opt-stiefel} detailed in Algorithm~\ref{alg:armijo}. We treat $\mathrm{St}(d, \ell)$ as an embedded sub-manifold in the vector space $\R^{d \times \ell}$ and introduce appropriate notions of tangent space, projection, and gradient. We provide a cursory overview of the Stiefel geometry for a self-contained treatment. A familiar reader may skip this part.
  
First, the tangent space at a point on a sub-manifold is intuitively the plane tangent to the sub-manifold at that point. The normal space is the corresponding orthogonal complement. The tangent space and the normal space to $\mathrm{St}(d, \ell)$ at $V$ are given as $T_V \mathrm{St}(d, \ell) =\left\{Z \in \mathbb{R}^{d \times \ell}: V^\top Z+Z^\top V=0\right\}$ and $\left(T_V \mathrm{St}(d, \ell)\right)^{\perp} =\left\{V S: S = S^\top \in \mathbb{R}^{\ell \times \ell}\right\}$, respectively. We view the Stiefel manifold as a sub-manifold of $\R^{d \times \ell}$, with its tangent spaces inheriting the Euclidean metric (Frobenius norm) $\| \cdot \|_{\mathrm{F}}$ and inner product $\langle A, B \rangle = \mathrm{Tr}(A^\top B)$. Then, projections of $\xi \in \R^{d \times \ell}$ on to the tangent and the normal spaces at $V$ are given by $P_V (\xi)=  \left(I_d -V V^\top\right) \xi+V \mathrm{skew}\left(V^\top \xi\right)$ and $P_V^{\perp} (\xi) =V \mathrm{sym}\left(V^\top \xi\right)$, respectively, where $\mathrm{skew}(A)= (A - A^\top)/2$ and $\mathrm{sym}(A) = (A + A^\top)/2$.

Finally, if $F$ is a smooth function defined on $\R^{d \times \ell}$, and $\bar F$ is its restriction to the Riemannian sub-manifold $\mathrm{St}(d, \ell)$, the gradient of $\bar F$ is equal to the projection of the gradient of $F$ onto $T_V \mathrm{St}(d, \ell)$, namely, $\mathrm{grad} ~\bar F(V) = P_V(\mathrm{grad} ~F(V))$. A rigorous treatment of the above can be found in \citet[Chapters 3 and 4]{absil2008optimization}.

\paragraph{Gradient Formula and Retraction}
With the geometry understood, the Riemannian gradient of the objective~\eqref{eqn:objective} can be readily calculated.
\begin{proposition}[Riemannian Gradient and Retraction] 
	\label{prop:riem-grad}
    Consider $\bar \Phi_{\upsilon, \eta}: \mathrm{St}(d, \ell) \rightarrow \mathbb{R} $, the restriction of $\Phi_{\upsilon, \eta}$ to the Riemannian sub-manifold $\mathrm{St}(d, \ell)$. The Riemannian gradient of the objective~\eqref{eqn:opt-stiefel} is
	\begin{equation}
    \label{eqn:riem-grad}
		\mathrm{grad} ~\bar \Phi_{\upsilon, \eta}(V) = (I_d - V V^\top) \left((\Sigma_\cS V \alpha_V - \E_\cS[XY]) \alpha_V^\top  + \eta DVV^\top D V \right) \in  T_{V} \mathrm{St}(d,\ell) \;.
	\end{equation}
	Moreover, the gradient formula for the objective~\eqref{eqn:objective} is the same for either the Stiefel or the Grassmann manifold. 
\end{proposition}

For any $\xi \in T_{V} \mathrm{St}(d,\ell)$, the retraction map via matrix polar factorization is defined as
\begin{equation}
\label{eqn:retraction}
	\mathrm{R}_V(\xi) = (V + \xi) (I_\ell + \xi^\top \xi )^{-1/2} \;.
\end{equation}
Proposition \ref{prop:riem-grad} allows us to define first-order descent methods to optimize \eqref{eqn:opt-stiefel}. For example, a simple line-search method on the Stiefel manifold updates the iterate as $V_{k+1} = \mathrm{R}_{V_k}(t_k \cdot G_k)$ for $G_k = \mathrm{grad} ~\bar \Phi_{\upsilon, \eta}(V_k)$, where $\mathrm{R}_{V_k}$  is the retraction map from~\eqref{eqn:retraction} and $t_k$ is a step-size. Now, all ingredients of Algorithm~\ref{alg:armijo} readily unfold.

\subsection{Stability, Predictability and Connections to Modern ML}
It was suggested in Section~\ref{sec:stiefel-opt} that the penalized objective \eqref{eqn:opt-stiefel} is an instantiation of a robustness and accuracy tradeoff. Observe that under the linear SCM in Definition~\ref{def:SCM-model}, two notions of stability/invariance coincide: optimizing over a \textit{stable representation} amidst covariate shifts will align with searching for the \textit{causal, invariant relationship}, stable across environments. We make the stability vs.\ predictability tradeoff explicit now and, by doing so, relate the method to a class of procedures recently popularized in ML. Our procedure seeks a two-part solution: a linear projection $V$, enforcing a similarity between $\cP_\cS$ and $\cP_\cT$, and a linear estimator, $\alpha$, built atop. Selecting penalization then translates to a balancing of stability and predictability. This is transparent once we decouple $F_{\upsilon, \eta}$ into loss and regularization,
\begin{align*}
    F_{\upsilon, \eta}(V,\alpha) :=  \overbrace{  \tfrac{1}{2} \E_\cS[(Y - \langle X, V \alpha \rangle)^2]}^{\textit{Predictivity}} \ \ \ + \ \ \ \overbrace{  \tfrac{\upsilon}{2}   \|\alpha\|^2 + \tfrac{\eta}{4} \left\| V^\top \left(\Sigma_\cT  - \Sigma_\cS \right) V \right\|_{\mathrm{F}}^2}^{\textit{Stability}} \;.
\end{align*}

The source \textit{Predictivity} term corresponds to prediction under squared loss for the regression $Y \sim V^\top X$ confined to a subspace $V$; it extracts a linear relationship $\alpha$ between the label $Y$ and a predictive subspace of the data $V^\top X$.
 
The \textit{Stability} term quantifies two things. First, a notion of distributional stability determined by the penalty $\eta \left\| V^\top \left(\Sigma_\cT  - \Sigma_\cS \right) V \right\|_{\mathrm{F}}^2$ for the representation subspace $V$. The penalty can be interpreted as a distribution distance; see \eqref{eqn:MMD}. Second, a standard ridge penalization quantifies the stability of the linear relationship $\alpha$ given the subspace $V$.

Conceptually, as $\eta \to \infty$, $V$ captures the linear subspace over which the covariate second moments stay invariant. 
Moreover, such an invariant linear subspace corresponds to the contribution from the exogenous $Z$ that identifies the invariant causal relationship. 
Indeed, $\Theta^\top \left(\E_\mathcal{T}[X X^\top]  - \E_\mathcal{S}[X X^\top] \right) \Theta = 0$, therefore provided $l > k$, $\Theta \in \R^{d \times k}$ is an element of the non-empty solution set. In this sense, the case when $\eta \to \infty$ coincides with the invariant estimator explored in Proposition \ref{prop:TRI}. More generally, as we will see in Section \ref{sec:example}, a choice of $\eta$ in a certain interval may balance stability and source predictability to harness improvements in target risk, the primary object in domain adaptation.

It is easy to see how the framework can be generalized to deal with different notions of distributional discrepancy, loss function, and model class. Specifically, a family of methods is based on finding a stable and predictive representation $\phi \colon \cX \to \cZ$, where $\cZ$ is a suitable feature space. Let $\phi_\# \mu$ and $\phi_\# \nu$ denote the source and target covariate distributions pushed forward to the mapped space $\cZ$. The goal is to find a representation $\phi \colon \cX \to \cZ$ and a composite predictor $g \colon \cZ \to \cY$ simultaneously by solving $\min_{g, \phi}  R_{\cS}(g \circ \phi) + \eta \cdot d(\phi_\# \mu, \phi_\# \nu)$, where $d(\cdot, \cdot)$ is a suitable discrepancy between probability distributions on the mapped space $\cZ$. Several choices for $d$ have been proposed in the ML literature: \citet{tzeng2014deep} and \citet{long2015learning} use the kernel MMD with a suitable kernel on $\cZ$, \citet{ganin2016domain} uses a suitable hypothesis-based discrepancy introduced in \citet{ben2006analysis,ben2010theory}, and \citet{shen2018wasserstein} uses the Wasserstein-1 distance, to name a few.

\section{Theory: Invariance Alignment and Risk Stability}
\label{sec:theory}
We now study theoretical properties of the methodology proposed in Section~\ref{sec:methodology}: alignment with the invariant subspace and risk across environments. The first result in Theorem~\ref{thm:alignment-to-invariant-subspace} characterizes the landscape of the non-convex manifold optimization. We show almost all local optima exhibit favorable alignment with the invariant linear subspace as long as the regularization parameter is sufficiently large. The second result speaks to domain adaptation. Building on the landscape result, we derive a stability bound between the source and target risk in Theorem~\ref{thm:stability-source-target}, which provides theoretical underpinnings for the empirical phenomena observed in Section \ref{sec:example}. On top of these, Theorem~\ref{thm:finite-sample-error} analyzes a finite-sample estimation error by quantifying the gap between the optimization procedure \eqref{eqn:opt-stiefel} and its plug-in estimation.

\subsection{Alignment with the Invariant Subspace}
In Propositions~\ref{prop:DREI} and \ref{prop:TRI}, the invariant subspace $\Theta$ is shown to be advantageous. The following theorem proves that the data-driven manifold optimization method given in Section~\ref{sec:methodology} can identify linear subspaces that are approximately orthogonal to the contribution of the endogenous, confounding variable $E$. We employ the canonical correlation (or angle) between linear subspaces to quantify the notion of approximate orthogonality. By dodging the endogenous subspace where $E$ influences $X$, the learned subspace $V$ approximately aligns with the invariant subspace---crucial for concept invariance and distributional stability.
 
\begin{theorem}[Alignment with the Invariant Subspace]
	\label{thm:alignment-to-invariant-subspace}
	Consider the model in Definition~\ref{def:SCM-model} that satisfies Assumptions \ref{asmp:base} and \ref{asmp:richer-target}, and $\Delta^\top \Delta = I_r$. Recall that $D := \Sigma_\cT - \Sigma_\cS$ denotes the matrix of covariance shift. Suppose $V \in \mathrm{St}(d, \ell)$ satisfies the following conditions:
	\begin{enumerate}
		\setlength\itemsep{0em}
		\item $V$ is a first-order stationary point of the optimization \eqref{eqn:opt-stiefel} with $\mathrm{grad} ~\Phi_{\upsilon, \eta}(V) = 0$,
		\item $V$ is in the admissible set $\mathcal{S}_{\Delta}(\delta)$, where $\mathcal{S}_{\Delta} := \{ V \in \mathrm{St}(d, \ell) ~:~ \| V^\top \Delta \|^2_{\mathrm{op}} \leq 1-\delta \}$.
	\end{enumerate}
	Then, we have
	\begin{equation*}
		\| V^\top \Delta \|_{\mathrm{op}}^6 \leq \frac{\delta^{-1} \lambda_{\max}(\Sigma_\cS) \| \Sigma_\cS^{-1/2} \E_\cS[XY] \|^4}{\lambda_{\min}(\Delta^\top D \Delta)^4} \frac{1}{4\upsilon \eta^2} \;.
	\end{equation*}
\end{theorem}

\begin{remark}
	Theorem \ref{thm:alignment-to-invariant-subspace} characterizes the optimization landscape on the Stiefel manifold and should be interpreted when the regularization parameter $\eta$ is large. It has two noteworthy features. First, we establish a structural result for \textit{all local optima} for the non-convex manifold optimization, which can be found efficiently in practice using Algorithm~\ref{alg:armijo}; we do not require the solution $V$ to be the \text{global optima}. Second, the result operates even when the invariant subspace overlaps with the endogenous space, namely $\Theta^\top \Delta \neq 0$. In such a case, the learned subspace $V$ will still be approximately orthogonal to $\Delta$, in the sense that $\max_{\bv \in V, \bw \in \Delta} \cos \angle (\bv, \bw) \leq \mathrm{const.} \times \frac{1}{(\upsilon \eta^2)^{1/6}}$. Conceptually, the optimization procedure will identify a strict subspace inside the invariant subspace to approximately dodge the endogenous subspace $\Delta$, so long as the stability regularization parameter $\eta$ is large and ridge regularization parameter $\upsilon$ is not too small.
\end{remark}
	
\subsection{Stability and Target Risk Bound}

Given the alignment with the invariant subspace, what remains to be shown is its implication for the domain adaptation problem, a task for this section. In that regard, the following theorem derives a stability bound between the target and the source risks. Recall the estimator $\beta^{\upsilon, \eta} = V \alpha_V$ with $\alpha_V$ as in \eqref{eqn:inner-ridge}. The following theorem shows the stability of any stationary point $V$ of the non-convex Stiefel manifold optimization \eqref{eqn:opt-stiefel}.

\begin{theorem}[Stability between Source and Target]
\label{thm:stability-source-target}
	Consider the setting as in Theorem~\ref{thm:alignment-to-invariant-subspace}.
	Let $V \in \mathrm{St}(d, \ell)$ be any first-order stationary point of the optimization procedure \eqref{eqn:opt-stiefel} with regularization parameters $\upsilon, \eta$, and $\beta^{\upsilon, \eta}$ be the corresponding linear subspace estimator restricted to $V$. For any $\epsilon >0$, define 
	\begin{equation*}
		\mathsf{S}_{\epsilon, \delta} := (1+\epsilon^{-1}) \delta^{-1/3} \lambda_{\max}(\Sigma_{\cS})^{1/3}  \frac{\lambda_{\max}(\Delta^\top D \Delta)}{\lambda_{\min}(\Delta^\top D \Delta)^{4/3}} \| \Sigma_\cS^{-1/2} \E_\cS[XY]  \|^{10/3}  \;,
	\end{equation*}
	the following generalization bound holds
	\begin{equation*}
		R_\cT(\beta^{\upsilon, \eta}) - R_\cS(\beta^{\upsilon, \eta}) \leq (1+\epsilon) \cdot \langle \beta^\star + \Delta \gamma , D (\beta^\star + \Delta \gamma) \rangle + \mathsf{S}_{\epsilon, \delta} \cdot \frac{1}{(4\upsilon)^{4/3} \eta^{2/3} } \;.
	\end{equation*}
\end{theorem}
\begin{remark}
	Note that the first term is necessary. Even for the oracle $V$ that is perfectly invariant to the environment shift, the oracle gap between the target and source risks equals $ \langle \beta^\star + \Delta \gamma, D (\beta^\star + \Delta \gamma) \rangle$, in view of Proposition~\ref{prop:equality}. The above result proves an approximate oracle to the gap between target and source risk, quantified by $\epsilon$. As $\eta \rightarrow \infty$ (holding all else fixed), the stability bound $S_{\epsilon, \delta} \cdot \tfrac{1}{(4\upsilon)^{4/3} \eta^{2/3} }$ decreases with $\eta$, which confirms the numerical findings in Section~\ref{sec:example}. Conceptually, as $\eta$ increases, the source risk $R_{\cS}(\cdot)$ will increase as one trades off prediction power for stability; on the plus side, instability $R_\cT(\cdot) - R_\cS(\cdot) $ is reduced due to invariance. 
\end{remark}

\subsection{Finite-Sample Error Analysis}
In practice, we have access to labeled and unlabeled samples from the source and target domains, respectively. Accordingly, we approximate the optimization \eqref{eqn:opt-stiefel} by replacing $\Phi_{\upsilon, \eta}$ with the plug-in analog $\widehat{\Phi}_{\upsilon, \eta}$; see \eqref{eq:plug-in-estimator} below. A natural statistical question is whether the plug-in estimator $\widehat{\Phi}_{\upsilon, \eta}$ is a good approximation to the true objective function $\Phi_{\upsilon, \eta}$. To answer this, we bound the uniform deviation of the plug-in estimator $\widehat{\Phi}_{\upsilon, \eta}$ from the true objective function $\Phi_{\upsilon, \eta}$ over the Stiefel manifold $\mathrm{St}(d, \ell)$. This allows us to quantify the gap between the objective value $\Phi_{\upsilon, \eta}$ evaluated at the empirical minimizer, namely, a minimizer of $\widehat{\Phi}_{\upsilon, \eta}$, and the minimum of the population optimization \eqref{eqn:opt-stiefel}.

\begin{theorem}
	\label{thm:finite-sample-error}
	Let $\{(x_i, y_i)\}_{i = 1}^{n}$ and $\{\tilde{x}_i\}_{i = 1}^{n}$ be independent samples from $\cP_\cS(X, Y)$ and $\cP_\cT(X)$, respectively. Suppose that the supports of $\cP_\cS(X, Y)$ and $\cP_\cT(X, Y)$ are contained in $\{(x, y) \in \R^d \times \R: \|x\| \le M, |y| \le M\}$ for some $M > 0$. Let $\widehat{\Phi}_{\upsilon, \eta}$ be the plug-in estimator of $\Phi_{\upsilon, \eta}$ based on the samples: with the sample covariance matrices $\widehat{\Sigma}_\cT, \widehat{\Sigma}_\cS$, define
	\begin{equation}
		\label{eq:plug-in-estimator}
		\widehat{\Phi}_{\upsilon, \eta}(V) := \min_{\alpha \in \R^\ell} \tfrac{1}{2} \{ \tfrac{1}{n} \sum_{i = 1}^{n} (y_i - \langle V \alpha, x_i \rangle)^2 + \upsilon \|\alpha \|^2 + \tfrac{\eta}{2} \|V^\top ( \widehat{\Sigma}_\cT - \widehat{\Sigma}_\cS ) V\|_{\mathrm{F}}^2 \} \;.
	\end{equation}
	Let $\widehat{V} \in \mathrm{St}(d, \ell)$ be a minimizer of $\widehat{\Phi}_{\upsilon, \eta}$, that is, $\widehat{V} \in \argmin_{V \in \mathrm{St}(d, \ell)} \widehat{\Phi}_{\upsilon, \eta}(V)$. Then, for any $\delta \in (0, 1)$, the following inequality holds with probability at least $1 - \delta$:
	\begin{equation*}
		\Phi_{\upsilon, \eta}(\widehat{V}) - \min_{V \in \mathrm{St}(d, \ell)} \Phi_{\upsilon, \eta}(V) \le 9 (M + \frac{M^3}{\upsilon})^2 \sqrt{\frac{\log \frac{3}{\delta}}{n}} + \sC_{M, \Sigma_\cT, \Sigma_\cS} \cdot \eta \ell \left(\frac{\log \frac{6 d}{\delta}}{n} + \sqrt{\frac{\log \frac{6 d}{\delta}}{n}}\right) \;,
	\end{equation*}
	where $\sC_{M, \Sigma_\cT, \Sigma_\cS}$ is a constant depending on $M$, $\Sigma_\cT$, and $\Sigma_\cS$.
\end{theorem}

The idea of the proof of Theorem~\ref{thm:finite-sample-error} is to decompose the deviation $|\Phi_{\upsilon, \eta} - \widehat{\Phi}_{\upsilon, \eta}|$ into the deviation of the empirical risk from the population risk and the deviation of the empirical penalty term from the population penalty term. The first term on the right-hand side of the inequality of Theorem~\ref{thm:finite-sample-error} corresponds to the bound on the deviation of the risk, while the other term bounds the deviation of the penalty term. The complete proof of Theorem~\ref{thm:finite-sample-error} is provided in Appendix C in the Supplementary Material together with the exact form of the constant $\sC_{M, \Sigma_\cT, \Sigma_\cS}$.

\section{Real-World Data Examples}
\label{sec:example}
Having established a framework for the first-order optimization of \eqref{eqn:opt-stiefel}, we apply the method to several real world data sets exhibiting distribution shifts. The following examples naturally admit distinct environments with markedly different covariate distributions. We aim to verify that the causal structure we propose captures real-world relationships in data, and empirically explore the stability risk trade-off motivated in the preceding section.

\subsection{Revisiting the Motivating Example of Section \ref{sec:motivating-example}}
\label{sec:revisiting_motivating_example}
We begin by revisiting the example from the introduction. In his work, Tabellini argues \textit{Economic Output} in the 1990s is causally related to various measures of social and economic prosperity. These complex causal relationships lead to a rich data setting for domain adaptation. We consider two environments---historically rural areas (source) and urban areas (target)---and aim to create a linear predictive model for \textit{Economic Output}. Tabellini's work suggests both covariate and concept shifts may be present. For example, historic \textit{Urbanization} levels likely impacted other features (e.g., `Culture' in the 1990s), and the response variable, \textit{Economic Output}.

A simple regression may be problematic if we want a linear predictor that performs well in both rural and urban areas (Proposition~\ref{prop:conceptshift}). The issue is exacerbated when we lack equal access to information across the urbanization spectrum. Suppose we have access to rural data (\textit{Urbanization} $< 2\%$), but only limited labeled data from urban regions (\textit{Urbanization} $> 10\%$). We can cast this problem in the framework of Figure \ref{fig:model_diagram}---with $X$ and $Y$ as below, while we take \textit{Urbanization} as an environment variable (a component of $E$):
\begin{equation*}
    X = \{\textit{Institutions}, \textit{Literacy}, \textit{Enrollment}, \textit{Culture}, \textit{Urbanization}\}, \, Y = \{\textit{Economic Output}\} \;.
\end{equation*}
Borrowing from Tabellini's work, we would expect indicators of historical \textit{Institutions} (pre 1850), \textit{Literacy} (1880s), and \textit{Enrollment} rates (1960s) to be minimally affected by \textit{Urbanization} levels in the 1850s. Hence, we postulate the existence of a lower-dimensional variable $Z$ with $k = 3$ that is invariant to the shift of \textit{Urbanization}.

\begin{figure}[!b]
	\centering
    \begin{subfigure}[b]{0.42\textwidth}
        \centering
        \includegraphics[width=\textwidth]{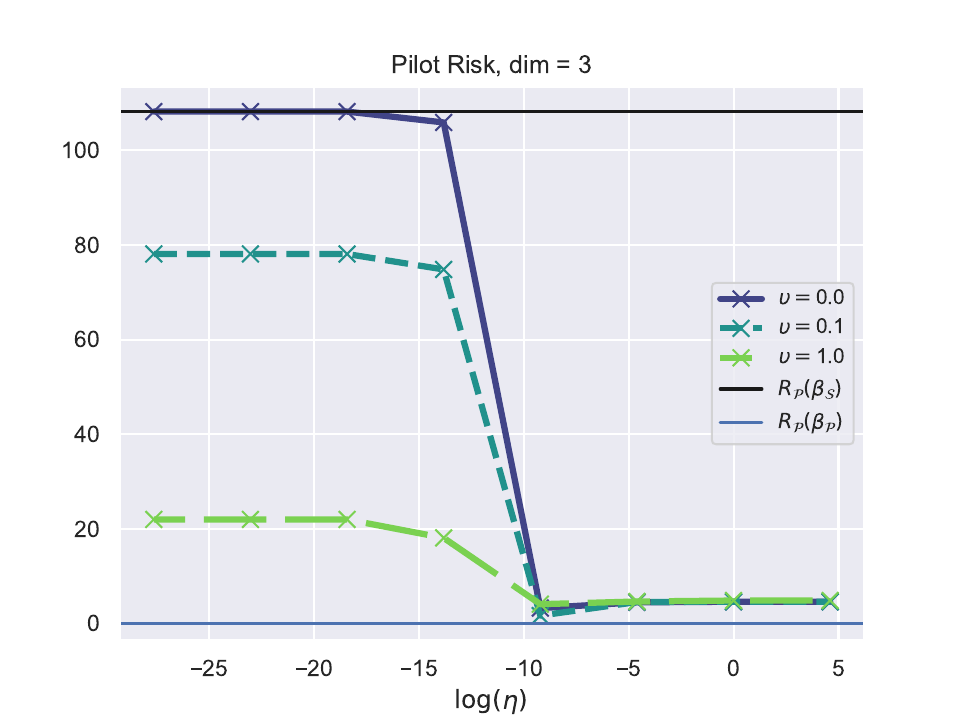}
        \caption{Pilot Risk}
    \end{subfigure}
    \begin{subfigure}[b]{0.42\textwidth}
        \centering
        \includegraphics[width=\textwidth]{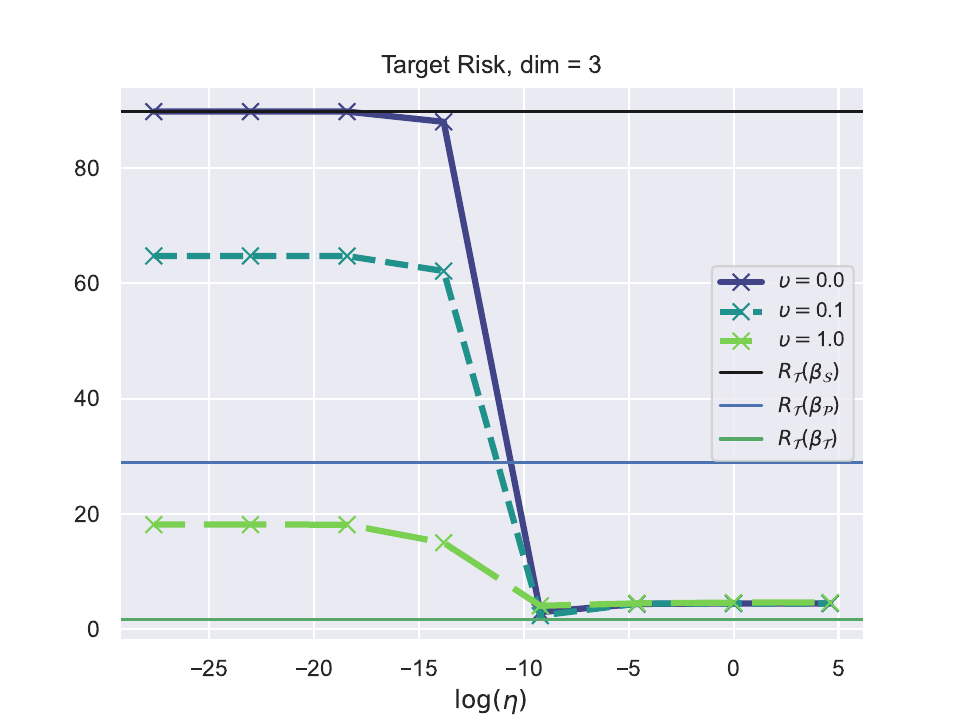}
        \caption{Target Risk}
    \end{subfigure}\\
    \centering
    \begin{subfigure}[b]{0.42\textwidth}
        \centering
        \includegraphics[width=\textwidth]{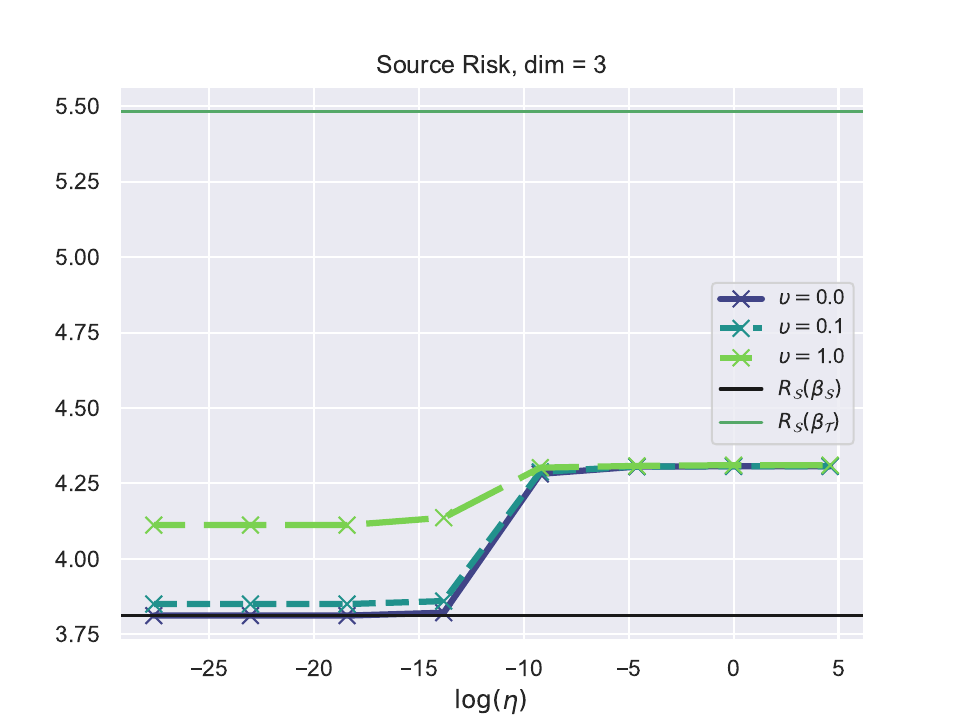}
        \caption{Source Risk}
    \end{subfigure}
    \begin{subfigure}[b]{0.42\textwidth}
        \centering
        \includegraphics[width=\textwidth]{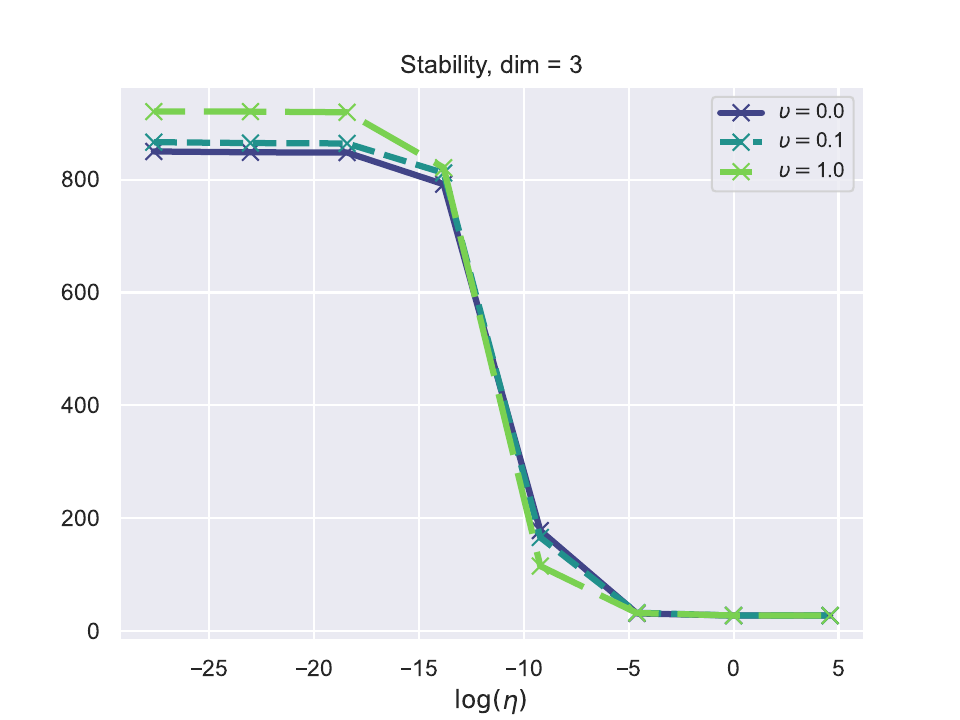}
        \caption{Stability}
    \end{subfigure}
    \caption{Performance of the linear subspace predictor $V \alpha_V$ obtained by solving \eqref{eqn:opt-stiefel} for different values of $(\upsilon, \eta)$ with $d = 5$ and $\ell = 3$. The subplots (a)-(c) plot the risk of the estimator on Pilot, Target, and Source data, respectively. (d) displays $\|V^\top(\Sigma_{\cT} - \Sigma_{\cS})V\|_{\mathrm{F}}$, a quantifier for the stability. In (a)-(c), the black line tracks the risk of $\beta_\cS$, the blue line tracks Pilot Data OLS, and the green line tracks the risk of $\beta_\cT$---a best possible oracle benchmark infeasible in practice.}
    \label{fig:Institutions-Risk}
\end{figure}

\begin{figure}[!hb]
    \centering
    \includegraphics[trim=0.5cm 0.2cm 0.5cm 1.5cm, clip=true, width=0.99\linewidth]{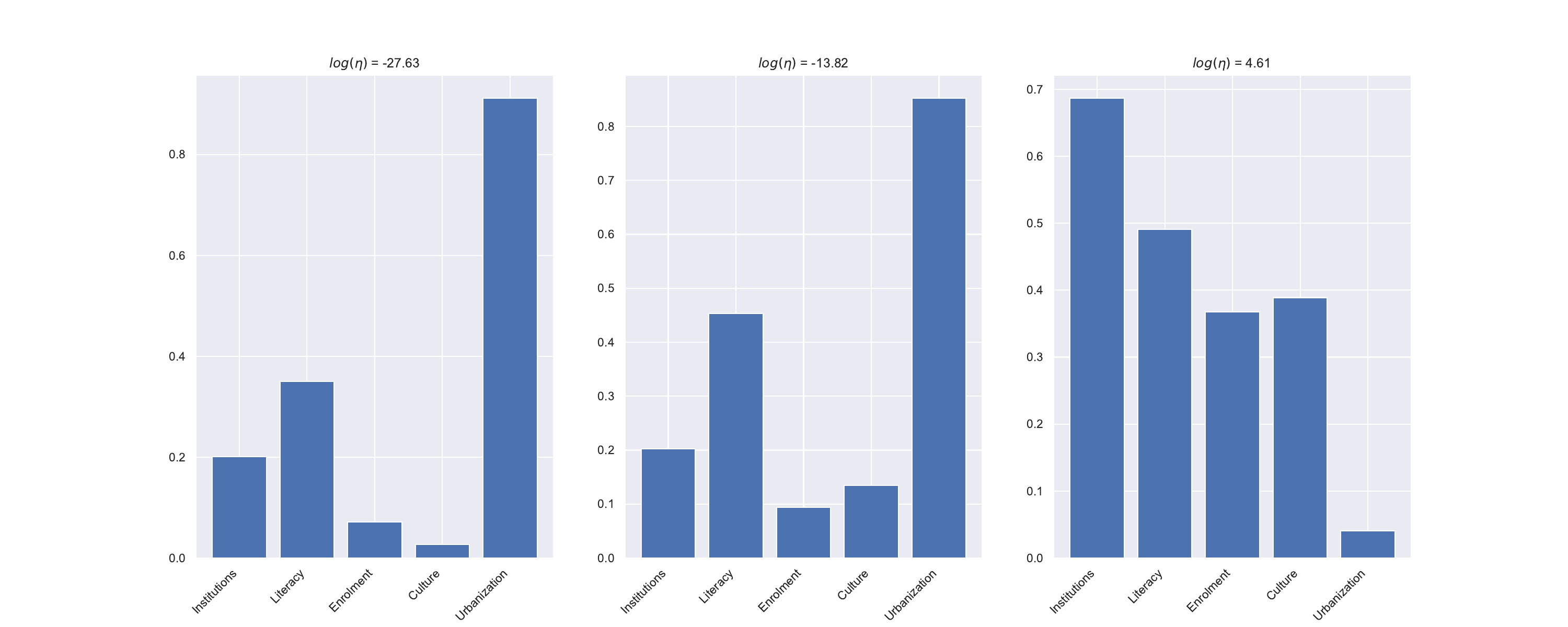}
    \caption{Results for $\upsilon = 0$ and $\ell = 1$ with $\log \eta \in \{-27.63, -13.82, 4.61\}$. We track the weighting placed on each feature by plotting the absolute value of each coefficient in the found $V \in \mathbb{R}^{d \times 1}$. As we increase $\eta$ (from left to right), greater emphasis is placed on Institutions, the valid instrument.}
    \label{fig:Institutions-Subspace}
\end{figure}

As our goal is prediction, not identification, we do not assume knowledge of any causal structure beyond $X \to Y$. Particularly, we do not suppose \textit{Institutions} act as a valid instrument, and do not preempt its role as $Z$. Instead, we run our method, obtaining a linear subspace predictor $V \alpha_V$ by solving \eqref{eqn:opt-stiefel} for different regularization parameters $\upsilon, \eta$ and Stiefel dimension $\ell$. For hyperparameter selection, and in line with our data constraints, we use a small amount of pilot data for validation ($\approx 20 \%$ of target data is labeled).

We find that $\ell = 3$ is an optimal lower-dimensional projection through pilot data validation. Figure \ref{fig:Institutions-Risk} illustrates the performance of the obtained predictor with this specification. Figure \ref{fig:Institutions-Risk} consists of 4 subplots: (a) the pilot risk, accuracy of the obtained predictor $V \alpha_V$ on a minuscule amount of labeled target data, (b) the target risk $R_\cT(V \alpha_V)$, the primary objective that is inaccessible, (c) the source risk $R_\cS(V \alpha_V)$, and (d) $\|V^\top(\Sigma_{\cT} - \Sigma_{\cS})V\|_{\mathrm{F}}$, a quantifier for the stability/invariance. 

While we have assumed no knowledge of the underlying causal structure, by navigating the stability-risk tradeoff, we find synergy with an instrumental variable perspective. With a larger stability parameter $\eta$, the lower dimensional subspace prioritizes \textit{Institutions} as a stable feature. This is emphasized in Figure \ref{fig:Institutions-Subspace}, where we optimize for a 1-dimensional subspace and track the direction of the found projection with respect to each feature. While we do not claim that maximizing for stability recovers instrumental features, we certainly expect features unaffected by environmental influence to be prominent in the dominant eigen-directions of $V$ as $\eta \to \infty$. For this example, maximizing stability improves target predictive power. We will see this is not always the case. In general, aiming for target prediction involves navigating a trade-off between source prediction and stability.

\subsection{ML Predictions under Distribution Shifts}
\label{sec:unstructured_real_data}
Next, we apply our method to three real-world datasets with less transparent causal relationships. These examples represent typical ML prediction challenges and provide a valuable benchmark for testing the robustness of our method in the absence of rigorously justified SEMs. Additionally, we introduce a data-driven heuristic for hyperparameter selection that does not depend on pilot data.

\begin{itemize}
	\setlength\itemsep{0em}
    \item \textbf{Forest Fires Data} \citep{cortez2007data}: 
	The goal is to predict the burned area of forest fires in the northeast region of Portugal using various meteorological features. We choose seven covariates ($d = 7$): temperature, relative humidity, wind speed, precipitation, and three fire weather indices. We emulate seasonal shifts by considering two environments separated in time; the data corresponding to June, July, and August constitute the target, while we take the remaining data as the source.
    \item \textbf{Bike Sharing Data} \citep{fanaee2014event}: 
	Obtained from the bikeshare system called Capital Bikeshare serving Washington, D.C., USA, the goal of this data set is to predict hourly counts of bike rentals using features representing weather information. We postulate two different environments depending on the seasons, spring and fall for source and target, respectively. Here, we take a subset of this data by focusing on workdays and take six features ($d = 6$): hour (0 to 23), temperature, feeling temperature, humidity, wind speed, and an indicator for working day.
    \item \textbf{Wine Quality Data} \citep{cortez2009modeling}: 
	This exercise aims to predict the quality of wine using eleven physicochemical features ($d = 11$): fixed acidity, volatile acidity, citric acid, residual sugar, chlorides, free sulfur dioxide, total sulfur dioxide, density, pH, sulphates, and alcohol. We use the data corresponding to white and red wine as the source and target, respectively.
\end{itemize}
Unlike the example in Section~\ref{sec:revisiting_motivating_example}, there are no well-studied models or instruments for these data sets. Accordingly, there is no clear choice for the hyperparameters: $\ell$ (Stiefel dimension), $\eta$ (stability parameter) and $\upsilon$ (ridge parameter). It is recommended to use a small amount of pilot data for validation if available. However, when pilot data is unavailable and we have no prior knowledge of the causal structure, it is inevitable to invoke a certain rule of thumb for hyperparameter selection.

\paragraph{Data-Driven Hyperparameter Selection} For such cases, we provide the following guidelines for hyperparameter selection. First, a principled way to choose the Stiefel dimension $\ell$ is to use principal component analysis (PCA). To find a stable subspace, it is reasonable to apply PCA to the pooled covariates from both source and target and choose the number of principal components that explain most of the variance; concretely, we recommend the smallest $\ell$ such that the cumulative explained variance ratio exceeds a certain threshold, say $0.9$. For the regularization parameters $\eta, \upsilon$, we recommend scaling them by balancing the three terms in \eqref{eqn:objective}. For the ridge parameter $\upsilon$, we first standardize the source covariates and then choose $\upsilon$ based on the ratio between the minimum source risk and the squared norm of the source risk minimizer, namely, $\frac{R_\cS(\beta_\cS)}{\|\beta_\cS\|^2}$. While letting $\upsilon \approx \frac{R_\cS(\beta_\cS)}{\|\beta_\cS\|^2}$ balances the source risk and the regularization term, we recommend a smaller value to avoid excessive regularization, say $\upsilon \approx \frac{0.1 \cdot R_\cS(\beta_\cS)}{\|\beta_\cS\|^2}$. For the stability parameter $\eta$, we first find $V \in \mathrm{St}(d, \ell)$ that minimizes \eqref{eqn:objective} without the prediction term, namely, $\min_{V \in \mathrm{St}(d, \ell)} \|V^\top (\Sigma_\cT - \Sigma_\cS) V\|_{\mathrm{F}}^2$. Then, we choose $\eta$ by balancing the obtained stability term and the minimum source risk, namely, $\eta \approx \frac{2 R_\cS(\beta_\cS)}{\|V^\top (\Sigma_\cT - \Sigma_\cS) V\|_{\mathrm{F}}^2}$. These data-driven guidelines, albeit heuristic, produce robust empirical performances, as shown below.

\paragraph{Results}
Figures \ref{fig:forest-fires}, \ref{fig:bike-sharing}, and \ref{fig:wine} illustrate the performance of the obtained predictor for the forest fires, bike sharing, and wine quality data, respectively. Each figure consists of three subfigures: (a) the target risk $R_\cT(V \alpha_V)$, the primary objective that is inaccessible, (b) the source risk $R_\cS(V \alpha_V)$, and (c) the difference $R_\cT(V \alpha_V) - R_\cS(V \alpha_V)$, a quantifier for the stability/invariance. For (a), the solid red horizontal line shows $R_\cT(\beta_\cS)$, the target risk of the vanilla source risk minimizer $\beta_\cS$, and the solid black line shows $R_\cT(\beta_\cT)$, the target risk of the target risk minimizer $\beta_\cT$ if we were given the labeled access to the target distribution $\cP_{\cT}(X, Y)$---a best possible oracle benchmark infeasible in practice.

\begin{figure}[!b]
	\centering
    \begin{subfigure}[b]{0.32\textwidth}
        \centering
        \includegraphics[trim=0.5cm 0.5cm 0.5cm 0.5cm, clip=true, width=\textwidth]{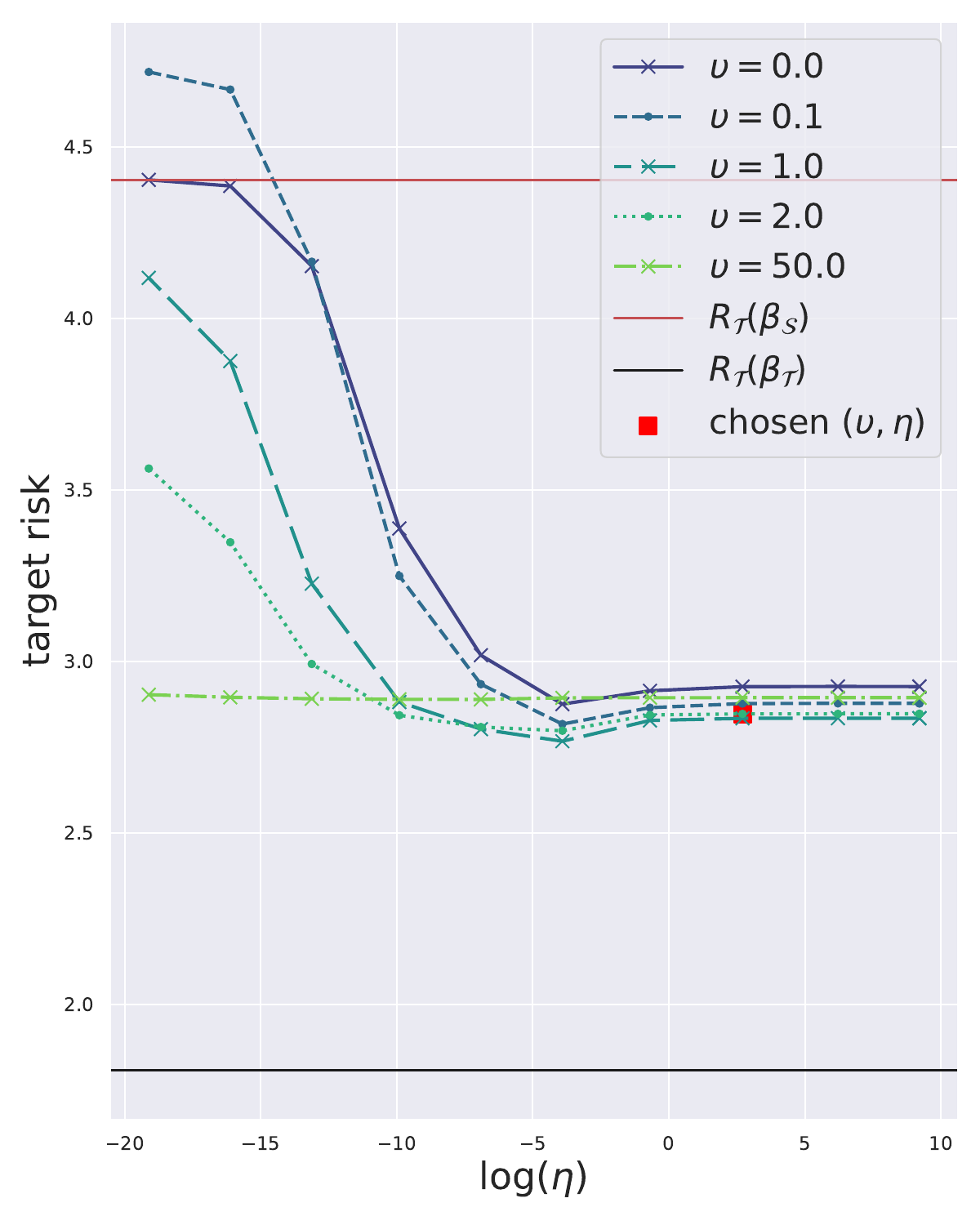}
        \caption{Target Risk}
    \end{subfigure}
    \begin{subfigure}[b]{0.32\textwidth}
        \centering
        \includegraphics[trim=0.5cm 0.5cm 0.5cm 0.5cm, clip=true, width=\textwidth]{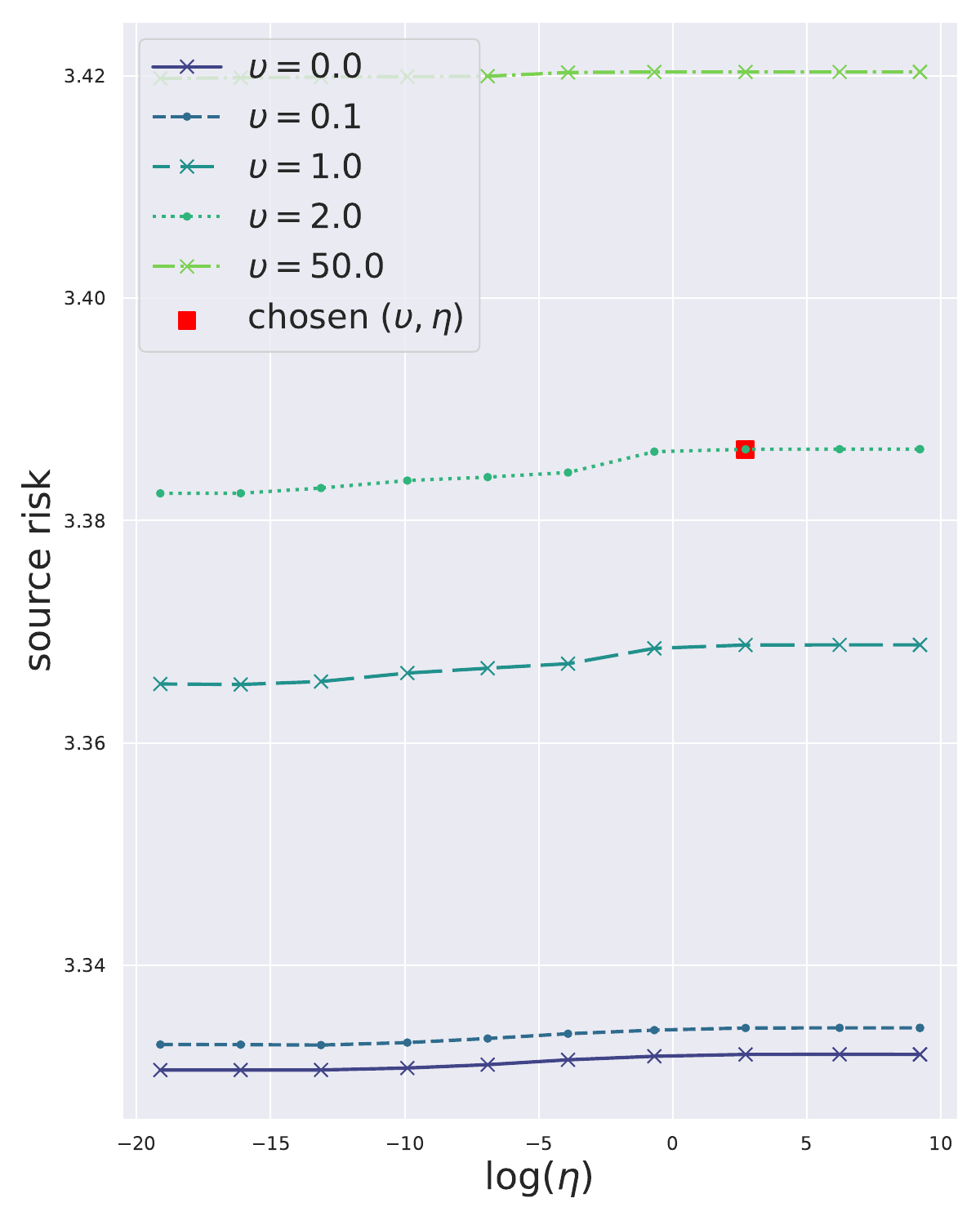}
        \caption{Source Risk}
    \end{subfigure}
    \begin{subfigure}[b]{0.32\textwidth}
        \centering
        \includegraphics[trim=0.5cm 0.5cm 0.5cm 0.5cm, clip=true, width=\textwidth]{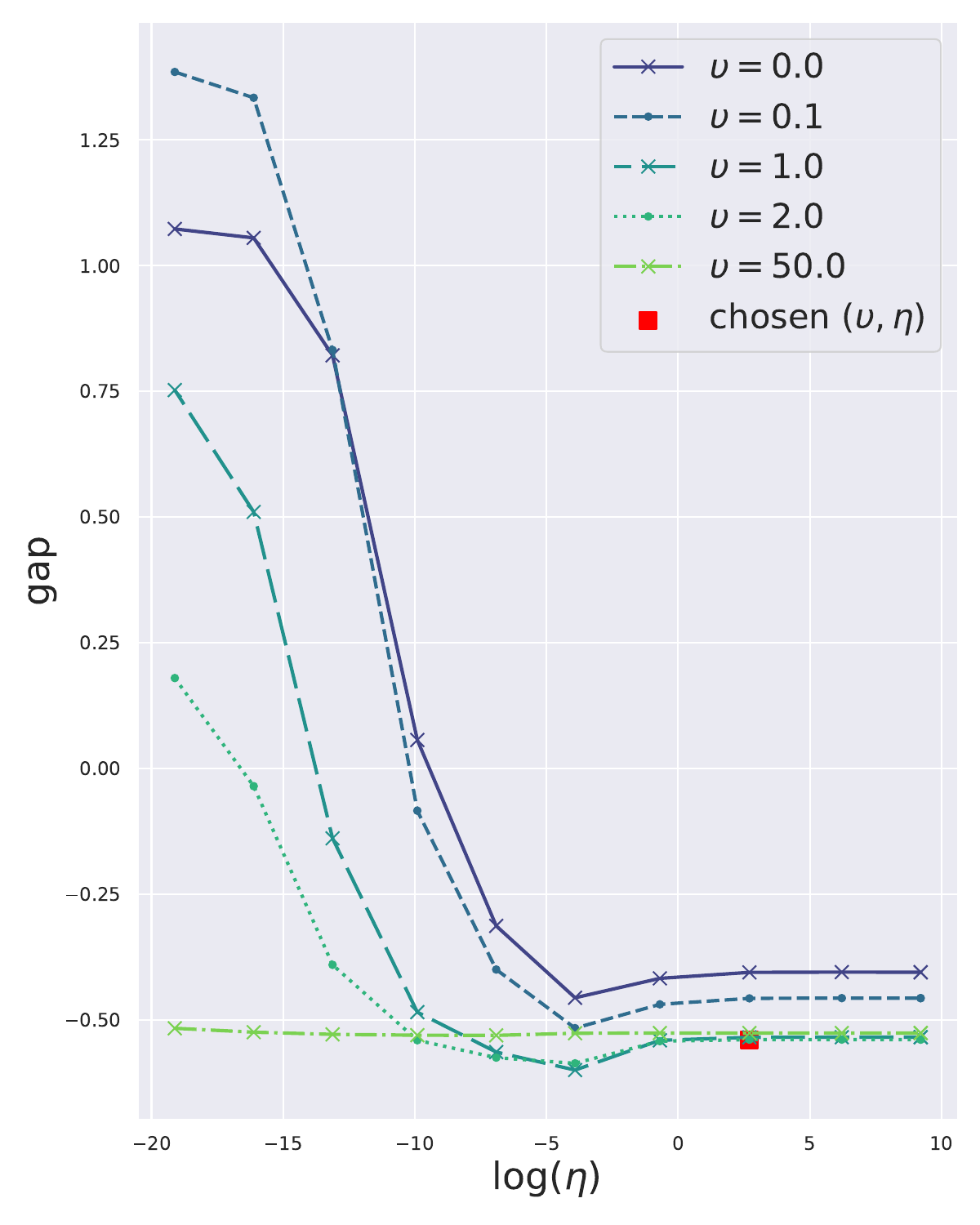}
        \caption{Target Risk - Source Risk}
    \end{subfigure}
	\caption{Forest Fires Data: Performance of the linear subspace predictor $V \alpha_V$ obtained by solving \eqref{eqn:opt-stiefel} for different values of $(\upsilon, \eta)$ with $d = 7$ and $\ell = 5$. (a) plots the risk on target dataset $R_\cT(V \alpha_V)$, where the solid red horizontal line shows $R_\cT(\beta_\cS)$ and the solid black line shows $R_\cT(\beta_\cT)$. (b) shows the risk on the source dataset $R_\cS(V \alpha_V)$. (c) plots the difference $R_\cT(V \alpha_V) - R_\cS(V \alpha_V)$.}
	\label{fig:forest-fires}
\end{figure}

\begin{figure}[!t]
	\centering
    \begin{subfigure}[b]{0.32\textwidth}
        \centering
        \includegraphics[trim=0.5cm 0.5cm 0.5cm 0.5cm, clip=true, width=\textwidth]{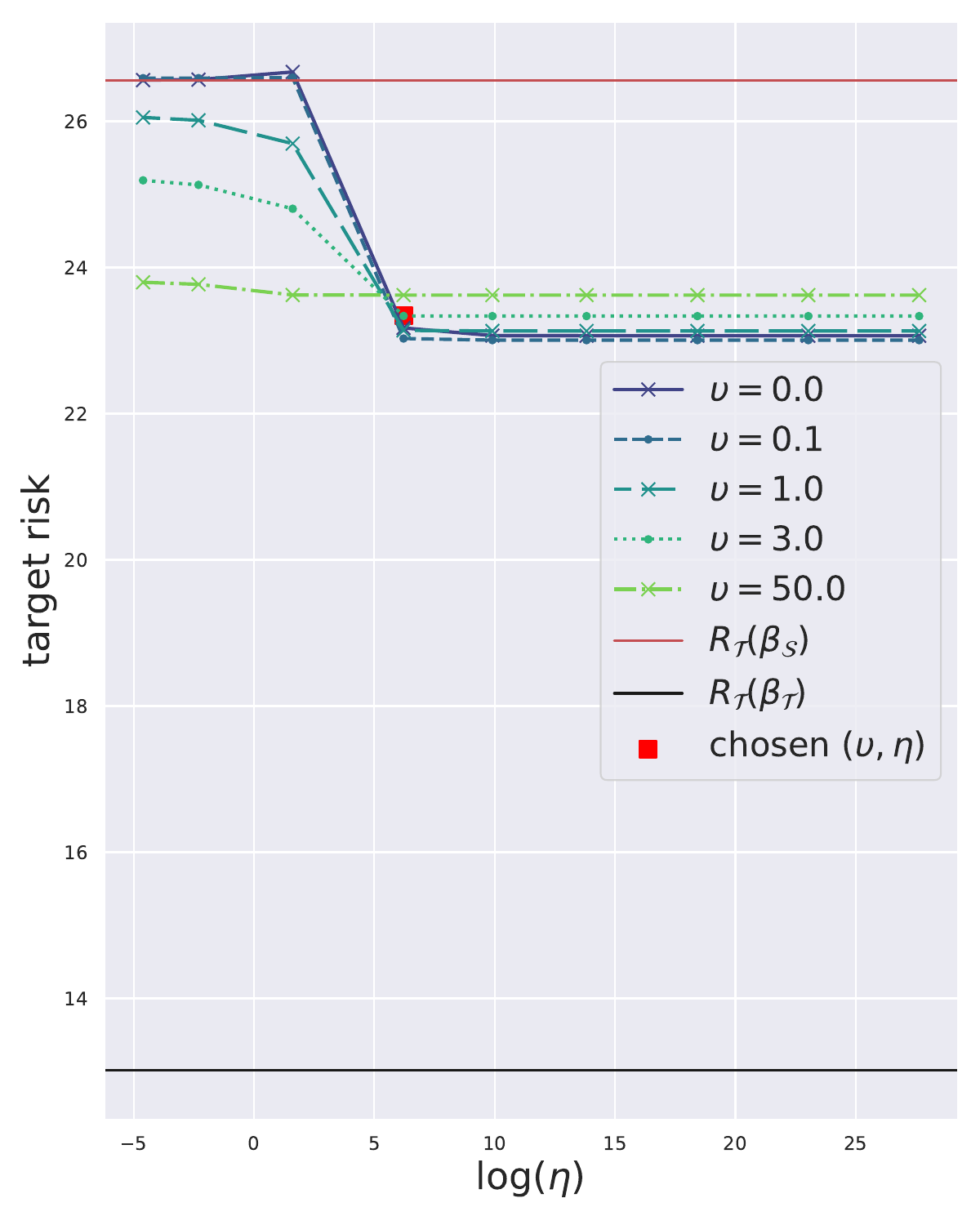}
        \caption{Target Risk}
    \end{subfigure}
    \begin{subfigure}[b]{0.32\textwidth}
        \centering
        \includegraphics[trim=0.5cm 0.5cm 0.5cm 0.5cm, clip=true, width=\textwidth]{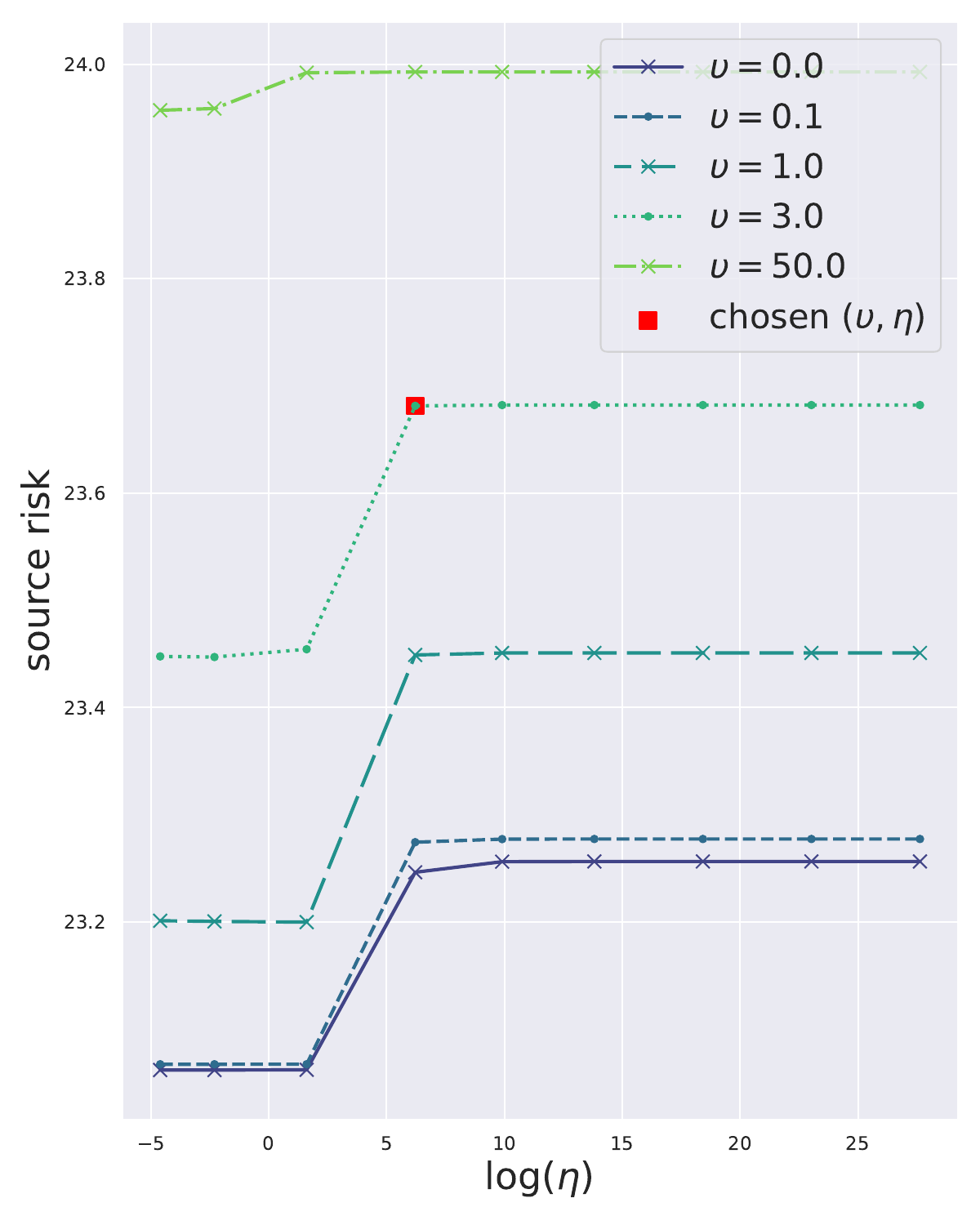}
        \caption{Source Risk}
    \end{subfigure}
    \begin{subfigure}[b]{0.32\textwidth}
        \centering
        \includegraphics[trim=0.5cm 0.5cm 0.5cm 0.5cm, clip=true, width=\textwidth]{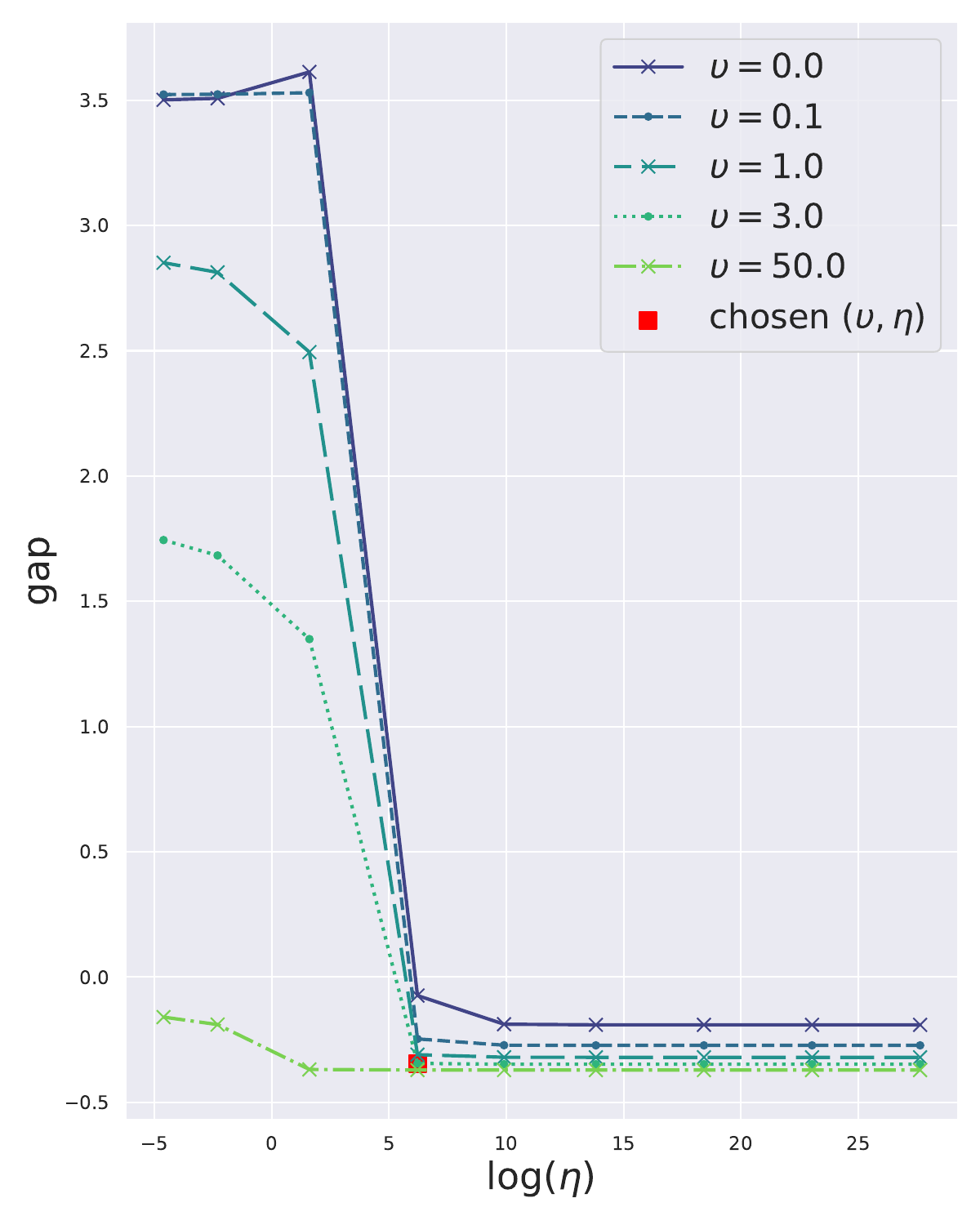}
        \caption{Target Risk - Source Risk}
    \end{subfigure}
    \caption{Bike Sharing Data: Performance of the linear subspace predictor $V \alpha_V$ obtained by solving \eqref{eqn:opt-stiefel} with $d = 6$ and $\ell = 5$. The subplots (a)-(c) plot the same quantities as in Figure~\ref{fig:forest-fires}.}
	\label{fig:bike-sharing}
\end{figure}

From Figures \ref{fig:forest-fires}(a), \ref{fig:bike-sharing}(a), and \ref{fig:wine}(a), we can see that for each $\upsilon$, there is a range of $\eta$ such that the resulting target risk, $R_\cT(V \alpha_V)$, is smaller than the target risk of the source risk minimizer $R_\cT(\beta_\cS)$ (visually below the red solid line). This indicates that the proposed procedure can improve the target risk by balancing the stability and the predictive power of the estimator for certain combinations of the parameters $\upsilon$ and $\eta$. Notably, for the wine quality data, we can see improvement over the source risk minimizer $\beta_\cS$ for any pair $(\upsilon, \eta)$, where we also achieve significant dimension reduction $\frac{\ell}{d} \approx 0.64$. Meanwhile, for the forest fires and bike sharing data, certain combinations of $(\upsilon, \eta)$ did not improve the target risk over the source risk. Based on the aforementioned hyperparameter selection guidelines, we obtain $(\upsilon, \eta) \approx (2, 15)$ for the forest fires data, $(\upsilon, \eta) \approx (3, 500)$ for the bike sharing data, and $(\upsilon, \eta) \approx (7, 1000)$ for the wine quality data. For these choices of hyperparameters, we find that the target risk is improved over the source risk for all data sets.

\begin{figure}[!t]
	\centering
    \begin{subfigure}[b]{0.32\textwidth}
        \centering
        \includegraphics[trim=0.5cm 0.5cm 0.5cm 0.5cm, clip=true, width=\textwidth]{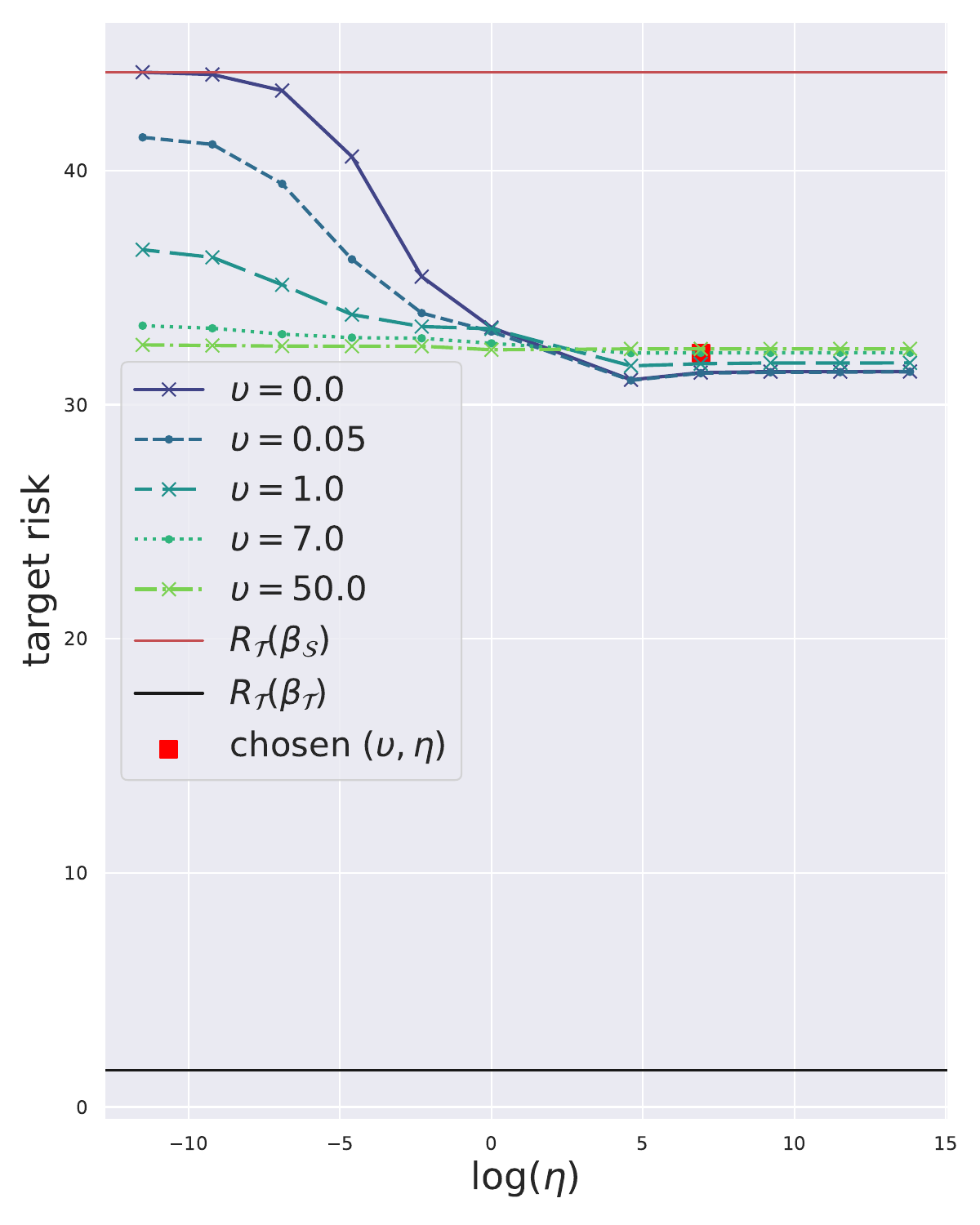}
        \caption{Target Risk}
    \end{subfigure}
    \begin{subfigure}[b]{0.32\textwidth}
        \centering
        \includegraphics[trim=0.5cm 0.5cm 0.5cm 0.5cm, clip=true, width=\textwidth]{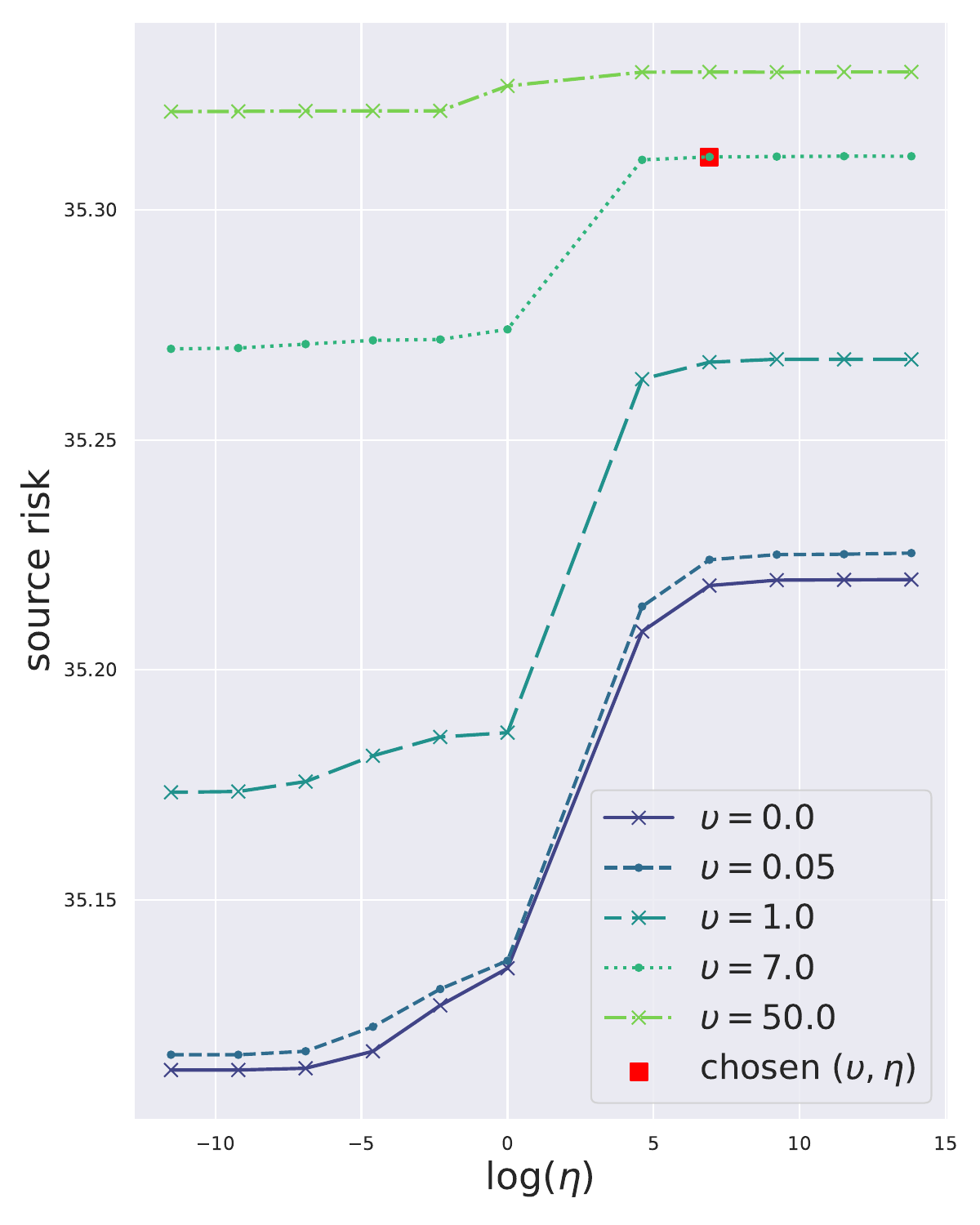}
        \caption{Source Risk}
    \end{subfigure}
    \begin{subfigure}[b]{0.32\textwidth}
        \centering
        \includegraphics[trim=0.5cm 0.5cm 0.5cm 0.5cm, clip=true, width=\textwidth]{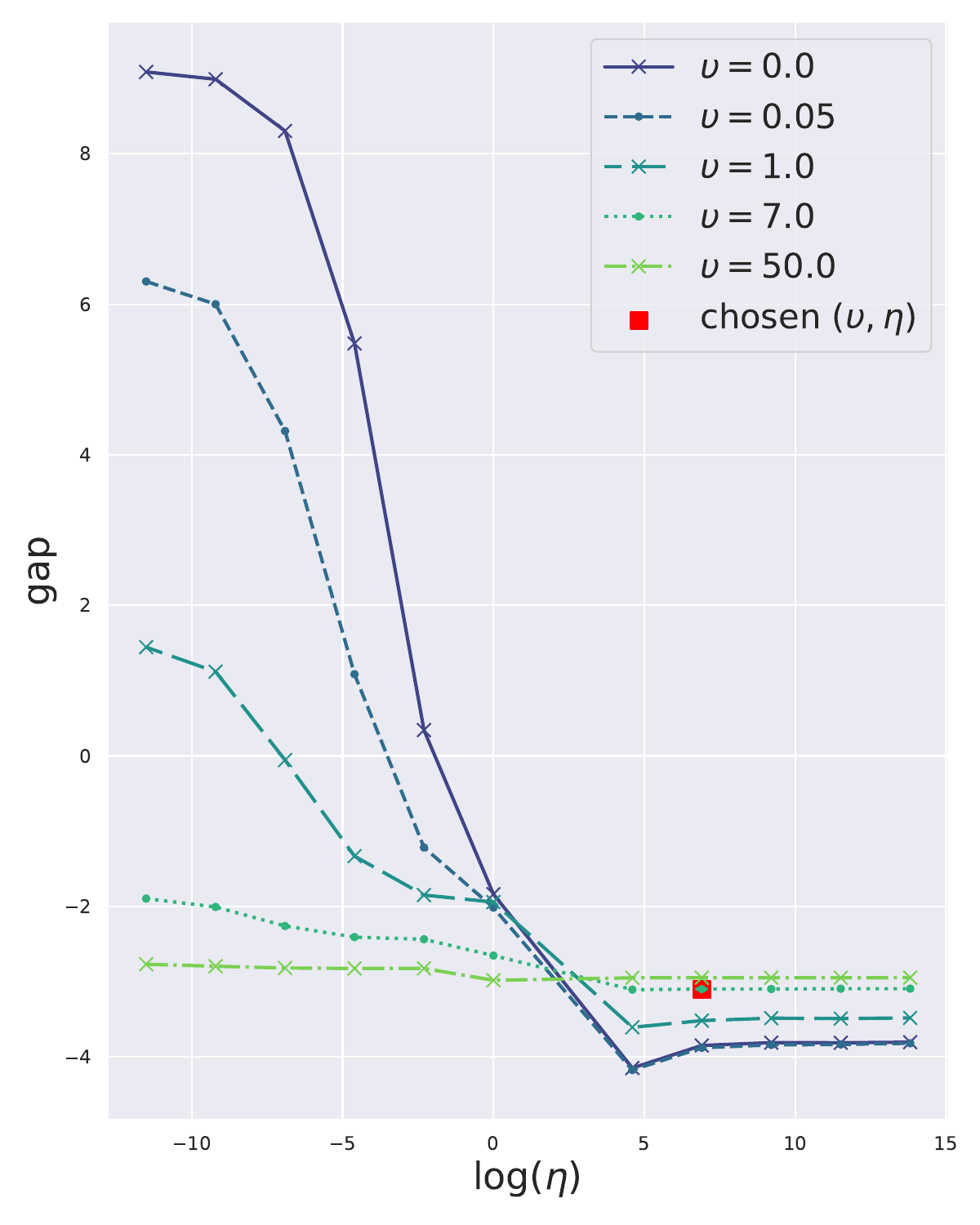}
        \caption{Target Risk - Source Risk}
    \end{subfigure}
    \caption{Wine Quality Data: Performance of the linear subspace predictor $V \alpha_V$ obtained by solving \eqref{eqn:opt-stiefel} with $d = 11$ and $\ell = 7$. The subplots (a)-(c) plot the same quantities as in Figure~\ref{fig:forest-fires}.}
	\label{fig:wine}
\end{figure}

Meanwhile, we observe the non-monotonic relationship between target risk and $\eta$ for fixed $\upsilon$, which is more pronounced in Figures \ref{fig:forest-fires}(a) and \ref{fig:wine}(a). This U-shaped behavior as a function of $\eta$ results from the conflicting monotone behaviors of source risk and the risk gap: source risks shown in Figures \ref{fig:forest-fires}(b) and \ref{fig:wine}(c) monotonically increase as $\eta$ increases, while the risk gap shown in the corresponding (c) is almost monotone decreasing, which is explored in Theorem \ref{thm:stability-source-target}. Figures \ref{fig:forest-fires}(c), \ref{fig:bike-sharing}(c), and \ref{fig:wine}(c) numerically validate the stability bound in Theorem \ref{thm:stability-source-target}: as $\eta$ increases, the risk gap tends to decrease in most cases.

\section{Conclusion}
This paper investigates domain adaptation for observational data while addressing the challenge of unobserved confounding. Unobserved confounding complicates domain adaptation by (i) altering the optimal statistical models across environments and (ii) extrapolating the model to the unseen domain of covariates. To address these challenges, we propose a causal model for distribution shifts in the presence of confounding. We highlight the limitations of traditional source risk minimization and the advantages of using an invariant subspace to stabilize learning and reduce target risk.

Methodologically, we introduce a domain adaptation algorithm that learns a representation through a lower-dimensional projection and tackles a non-convex manifold optimization problem using Riemannian optimization techniques. This algorithm acts as a stability-regularized source risk minimization that aligns with existing literature on stability and predictability in domain adaptation. Theoretically, we analyze the non-convex manifold optimization landscape and establish guarantees for nearly all local optima. With proper regularization, local optima from first-order Riemannian optimization will converge to an invariant linear subspace resistant to concept and covariate shifts. We also prove a target risk bound, showing that a predictive model derived from the learned subspace can achieve a nearly ideal gap between target and source risks. 

The low-dimensional projection our method seeks will implicitly balance invariance and source domain predictive power to minimize target risk. This will be governed by hyperparameters $(\eta, \upsilon)$ and Stiefel dimension $\ell$. Therefore, in practice, the choice of these hyperparameters is crucial. Hyperparameter selection guidelines in Section \ref{sec:unstructured_real_data} provide actionable heuristics, which we verify to be effective from the examples in Section \ref{sec:example}, but we believe more systematic hyperparameter selection rules could be more helpful. We leave this for future work. The current method is based on linear projections and linear regression; see Appendix B in the Supplementary Material for further discussions on the testability of the main assumption and comparison to the usual intrumental variable setting. It would be interesting to extend the proposed framework to partially linear or nonlinear models for full generality, as well as to the setting with multiple source environments. We leave these as future directions.

In summary, our method is a data-driven approach that extends the utility of instrumental variables for identification problems to domain adaptation by automatically finding the stable representation. Amid increasing interest in ML/AI, our method shows how to utilize traditional statistical/econometric techniques for data discovery in the AI context by properly designing automated procedures.

\section*{Acknowledgments}

Liang acknowledges the generous support from the NSF Career Grant (DMS-2042473), and the William Ladany Faculty Fellowship from the University of Chicago Booth School of Business.

\bibliography{reference.bib}
\bibliographystyle{plainnat}

\newpage
\begin{center}
  \textbf{\LARGE Supplementary Material for ``Learning When the Concept Shifts: Confounding, Invariance, and Dimension Reduction''}
\end{center}

\appendix
\section{Related Literature}
\label{sec:literature}
A recent line of research has adopted a causal perspective to formalize methods of statistical reasoning for shifting environments. Notably, the works identify that the invariance principle of causal relations can be leveraged to generalize across arbitrary interventions---a strong notion of robustness, see \citet{Scholkopf2012OnCA}, \citet{Magliacane2017DomainAB} and \citet{Christiansen2022}. The ontological relationship between a response and its direct causes is predictive and invariant, and can be shown to be optimal against arbitrarily strong interventions \citep{peters2016causal, RojasCarulla2015InvariantMF}. An influential sequence of papers \citep{peters2016causal, HeinzeDeml2017InvariantCP, Pfister2019InvariantCP} casts distribution shifts as the action of some intervention. It compares \textit{labeled data sets} across distinct experimental settings to identify these invariant causal conditionals. However, facing observational data, not all direct causes are observed or identifiable. The interventions might not be diverse or strong enough, which makes causal regression overly conservative. In reality, practitioners face the fundamental tradeoffs between causality and predictability. \citet{rothenhausler2021anchor} proposed an actionable methodology to probe the tradeoffs, combining instrumental variable regression (for causality) and ordinary least squares (for predictability), when \textit{access to certain exogenous variables} acting as instruments is available. While \citet{rothenhausler2021anchor} has similar high-level commonalities with our work in terms of invariance and predictability tradeoffs, a key distinction is that \citet{rothenhausler2021anchor} requires known instruments for identification, whereas we postulate the existence of low-dimensional invariant features that are not necessarily known.

In contrast, the current paper tackles the domain adaptation problem in observational data when (i) labels are unavailable in the target environment, (ii) unobserved confounding factors drive environment shifts, and (iii) there is no access to exogenous, instrumental variables to construct quasi-experiments for causal identification \citep{Card-Krueger-1994-AER, Angrist-Imbens-1994-JASA}. This is a setting less studied in the literature. Our setting inherits the quintessential difficulty of domain adaptation for observational data, the unobserved confounding, and further delineates its role lurking in the environment. In spirit, our work follows and builds upon the literature on improving out-of-domain generalization in ML with causal insights without experimental or quasi-experimental data.

In the learning-theoretic literature, out-of-domain generalization bounds show that provided a stable representation is found under which probability distributions on source and target domain are similar, a hypothesis with low source risk will perform well on target data \citep{ben2006analysis, blitzer2007learning, mansour2009domain, ben2010theory}. This line of work also pioneered the tradeoffs between representation stability and predictability: stable representations may not harness the maximum predictive power. While methodologies that simultaneously search for predictive and stable representations of data \citep{ blitzer-etal-2006-domain, tzeng2014deep, long2015learning, Arjovsky2019InvariantRM} are now ubiquitous, it is unclear under what data generating mechanisms they perform well. Furthermore, whether the practical optimization method obtains provable guarantees for the learned representation remains to be better understood \citep{Arjovsky2019InvariantRM}. 

We anchor the domain adaptation problem in an accessible linear SCM and delineate situations where invariance yields improvement. The methodology in this paper can be interpreted as a bare-bones instantiation of stability-regularized risk minimization \citep{tzeng2014deep, ganin2016domain, shen2018wasserstein} in the linear SCM. The concrete relationship between stability and predictability will be teased out. We also establish provable characterizations for the learned linear subspace, a concrete step toward representation learning. As we shall see, two notions of stability/invariance unite under the linear SCM: optimizing over a \textit{stable representation} in covariate shifts will align well with searching for the \textit{causal, invariant relationship}, stable across environment shifts.

The problem of covariate shifts, with the conditional concept $Y|X$ unchanged, has spurred considerable interest in statistics and ML/AI. The statistical study of covariate shift under parametric models dates back to the pathbreaking work of \citet{Shimodaira2000ImprovingPI}. \citet{Shimodaira2000ImprovingPI} established the asymptotic optimality of vanilla Maximum Likelihood Estimation (MLE) in the well-specified setting and that of weighted MLE under misspecification. Recent non-asymptotic results bolster this result \citep{ge2023maximum}, and find that the weighted MLE can be minimax optimal under misspecification. The above works operate under a mild covariate shift setting---namely, the likelihood ratio between source and target is often required to be bounded and estimable based on data. Several works explore different facets of covariate shifts specific to the linear setting. \citet{lei2021near} considers the minimax optimal estimator for linear regression under fixed design, where the learner can access some amount of unlabeled target data. However, in the presence of concept shift, near minimax optimal estimation requires target labels. \citet{mousavi2020minimax} provides lower bounds for out-of-distribution generalization in linear and one-hidden-layer neural network models. For hard covariate shifts with an unbounded likelihood ratio, \citet{Liang_2024} proposed to study adversarial covariate shifts to understand what extrapolation region adversarial covariate shifts will focus on for a given linear model. \citet{Liang_2024} derived a curious dichotomy: depending on the regression or the classification setting, adversarial covariate shift can either be a blessing or a curse to the subsequent learning. In comparison, the current paper anchors the problem in a linear SCM where confounding generates both concept and covariate shifts---diverging from the well-specified settings explored in the literature. Our method does not require the bounded likelihood ratio assumption nor the need to estimate such a ratio. Specific to the linear SCM, our paper addresses when and how access to unlabeled target samples can boost target domain performance in the presence of concept shifts.

\section{Discussion}
\subsection{Comparison to the Instrumental Variable Setting}
Reflecting on our model (see Figure \ref{fig:model_diagram}), similarities with the classic IV setting are immediate. Our invariant variable, $Z$, plays a similar role to an instrumental variable, in the sense that it is independent of the confounding environment influence, $E$. Our confounding variable, $E$, affects our observed covariates, $X$, through a restricted linear subspace. This assumption is not necessary in the IV setting. However, unlike the IV literature, we are in the observational setting where $Z$ and $E$ are not known, and in the absence of experiments, cannot necessarily be identified. The linear subspace assumption ensures we can leverage some low-dimensional, stable projection for our goal, which, unlike classical IV, is minimizing target domain risk. 

\subsection{Testability of Assumptions}
While it is not necessary (or possible in general) to identify $Z$ or $E$, our method does hinge on the assumption that covariate shift exists. As the proposed method is based on linear predictors, the covariate shift, if it exists, should be detected from the second moment matrices $\Sigma_\cE = \E_\cE[X X^\top]$ for $\cE \in \{\cS, \cT\}$. Hence, using empirical data, we could test for the existence of covariate shifts due to the environmental shift. Suppose we have two sets of i.i.d.\ covariate vectors from the source and target environments; for simplicity, denote them together as $x_1, \ldots, x_{n + m}$, where the first $n$ of them, $x_1, \ldots, x_n$, are from the source, while the others, $x_{n + 1}, \ldots, x_{n + m}$, are from the target. We compute the empirical second moment matrices:
\begin{equation*}
    \hat{\Sigma}_\cS : = \frac{1}{n} \sum_{i = 1}^{n} x_i x_i^\top \quad \text{and} \quad \hat{\Sigma}_\cT := \frac{1}{m} \sum_{i = n + 1}^{n + m} x_i x_i^\top.
\end{equation*}
We compute a suitable statistic comparing $\hat{\Sigma}_\cS, \hat{\Sigma}_\cT$, for instance, the operator norm of their difference $\|\hat{\Sigma}_\cS - \hat{\Sigma}_\cT\|_{\mathrm{op}}$, which we can compare to its permuted counterpart $\|\hat{\Sigma}_\cS^\sigma - \hat{\Sigma}_\cT^\sigma\|_{\mathrm{op}}$, where $\sigma$ is any permutation on $\{1, \ldots, n + m\}$ and 
\begin{equation*}
    \hat{\Sigma}_\cS^\sigma : = \frac{1}{n} \sum_{i = 1}^{n} x_{\sigma(i)} x_{\sigma(i)}^\top \quad \text{and} \quad \hat{\Sigma}_\cT^{\sigma} := \frac{1}{m} \sum_{i = n + 1}^{n + m} x_{\sigma(i)} x_{\sigma(i)}^\top.
\end{equation*}
In practice, we sample $B$ permutations $\sigma_1, \ldots, \sigma_B$ from the set of all permutations on $\{1, \ldots, n + m\}$ and see if $\|\hat{\Sigma}_\cS - \hat{\Sigma}_\cT\|_{\mathrm{op}}$ is larger than some quantile among $\{\|\hat{\Sigma}_\cS^{\sigma_b} - \hat{\Sigma}_\cT^{\sigma_b}\|_{\mathrm{op}}\}_{b = 1}^{B}$.
It is not necessary to detect which subspace contains this environment shift. We seek a low-dimensional projection best suited for target domain risk---this is not necessarily the space orthogonal to the shifting subspace.

\section{Technical Proofs}
\label{sec:proofs}
\subsection{Proofs in Section~\ref{sec:why-risk-min}}

\begin{proof}[Proof of Proposition~\ref{prop:conceptshift}]
	Under Assumption~\ref{asmp:base}, since $\tau > 0$ implies that $\E_\cE[X X^\top]$ is invertible, we can deduce that $\beta_\cE = (\E_\cE[X X^\top])^{-1} \E_\cE[X Y]$ is uniquely defined. One can verify that 
	\begin{equation}
		\label{eq:E[XY]}
		\E_\cE[X Y] = \E_\cE[X X^\top] \beta^\star + \Delta \E_\cE[E E^\top] \gamma + \Theta \E[Z] (\E_\cE[E])^\top \gamma = \E_\cE[X X^\top] \beta^\star + \Delta \E_\cE[E E^\top] \gamma \;.
	\end{equation}
	Hence, we have 
	\begin{equation*}
		\beta_\cE = (\E_\cE[X X^\top])^{-1} \E_\cE[X Y] = \beta^\star + (\E_\cE[X X^\top])^{-1} \Delta \E_\cE[E E^\top] \gamma \;.
	\end{equation*}
	Under Assumption \ref{asmp:base}, note that 
	\begin{equation}
		\label{eq:second_moment_based}
		\E_\cE[X X^\top]
		= 
		\begin{bmatrix}
			\Theta & \Delta 
		\end{bmatrix}
		\begin{bmatrix}
			\E[Z Z^\top] & 0 \\
			0 & \E_\cE[E E^\top]
		\end{bmatrix}
		\begin{bmatrix}
			\Theta^\top \\
			\Delta^\top
		\end{bmatrix}
		+ \tau^2 I_d \;.
	\end{equation}
	As $[\Theta, \Delta] \in \mathrm{St}(d, k + r)$, let $\Gamma \in O_{d-(k+r)}$ whose column space spans the orthogonal complement of $[\Theta, \Delta]$. Then, from \eqref{eq:second_moment_based}, we have
	\begin{equation}
		\label{eq:second_moment}
		\E_\cE[X X^\top]
		= 
		\begin{bmatrix}
			\Theta & \Delta & \Gamma
		\end{bmatrix}
		\begin{bmatrix}
			\E[Z Z^\top] + \tau^2 I_k & 0 & 0\\
			0 & \E_\cE[E E^\top] + \tau^2 I_r & 0 \\
			0 & 0 & \tau^2 I_{d - (k+r)}
		\end{bmatrix}
		\begin{bmatrix}
			\Theta^\top \\
			\Delta^\top \\
			\Gamma^\top
		\end{bmatrix} \;,
	\end{equation}
	where we use $\Gamma \Gamma^\top = I_d - \Theta \Theta^\top - \Delta \Delta^\top$. Therefore,
	\begin{equation}
		\label{eq:second_moment_inverse}
		\begin{aligned}
			& (\E_\cE[X X^\top])^{-1} \\
			& = 
			\begin{bmatrix}
				\Theta & \Delta & \Gamma
			\end{bmatrix}
			\begin{bmatrix}
				(\E[Z Z^\top] + \tau^2 I_k)^{-1} & 0 & 0 \\
				0 & (\E_\cE[E E^\top] + \tau^2 I_r)^{-1} & 0 \\
				0 & 0 & (\tau^2 I_{d - (k+r)})^{-1} \\
			\end{bmatrix}
			\begin{bmatrix}
				\Theta^\top \\
				\Delta^\top \\
				\Gamma^\top
			\end{bmatrix} \;.
		\end{aligned}
	\end{equation}
	Therefore, we have
	\begin{equation}
		\label{eq:beta_E}
		\begin{split}
			\beta_\cE 
			& = \beta^\star + (\E_\cE[X X^\top])^{-1} \Delta \E_\cE[E E^\top] \gamma \\
			& = \beta^\star + \Delta \left(\E_\cE[E E^\top] + \tau^2 I_r\right)^{-1} \E_\cE[E E^\top] \gamma \\
			& = \beta^\star + \Delta \gamma - \tau^2 \Delta \left(\E_\cE[E E^\top] + \tau^2 I_r\right)^{-1} \gamma \;,
		\end{split}
	\end{equation}
	where the second equality uses $(\E_\cE[X X])^{-1} \Delta = \Delta (\E_\cE[E E^\top] + \tau^2 I_r)^{-1}$ deduced from \eqref{eq:second_moment_inverse}.
	Now, if $\E_\cS[E E^\top] \neq \E_\cT[E E^\top]$, there must exist a nonzero coefficient $\gamma$ such that
	\begin{equation*}
		\tau^2 \Delta \left(\E_\cS[E E^\top] + \tau^2 I_r\right)^{-1} \gamma \neq \tau^2 \Delta \left(\E_\cT[E E^\top] + \tau^2 I_r\right)^{-1} \gamma \;,
	\end{equation*}
	which implies that $\beta_\cS \neq \beta_\cT$.
\end{proof}

\begin{proof}[Proof of Proposition~\ref{prop:DREI}]
	Note that $\alpha^\Theta_\cE = (\Theta^\top \E_\cE[X X^\top] \Theta)^{-1} \Theta^\top \E_\cE[X Y]$. From \eqref{eq:E[XY]}, we have $\Theta^\top \E_\cE[X Y] = \Theta^\top \E_\cE[X X^\top] \beta^\star$. From \eqref{eq:second_moment}, we have $\Theta^\top \E_\cE[X X^\top] = (\E[Z Z^\top] + \tau^2 I_k) \Theta^\top$. Therefore, 
	\begin{equation*}
		\begin{split}
			\alpha^\Theta_\cE = (\E[Z Z^\top] + \tau^2 I_k)^{-1} (\E[Z Z^\top] + \tau^2 I_k) \Theta^\top \beta^\star = \Theta^\top \beta^\star \;.
		\end{split}
	\end{equation*}
	As a result, $\beta^\Theta_\cE = \Theta \Theta^\top \beta^\star$, meaning that $\beta^\Theta_\cS = \beta^\Theta_\cT$ always holds regardless of $\gamma$. 
\end{proof}

\begin{proof}[Proof of Proposition \ref{prop:TRI}]
	Observe that
	\begin{equation*}
		R_\cT(\beta) = \|\beta_\cT - \beta\|_{\E_\cT[X X^\top]}^2 + \E_{\cT}[Y^2] - \| \beta_{\cT} \|_{\E_\cT[X X^\top]}^2 \;.
	\end{equation*}
	Hence, $R_\cT(\beta^\Theta_\cS) < R_\cT(\beta_\cS)$ is equivalent to $\|\beta_\cT - \beta^\Theta_\cS\|_{\E_\cT[X X^\top]}^2 < \|\beta_\cT - \beta_\cS\|_{\E_\cT[X X^\top]}^2$. From \eqref{eq:beta_E} and $\beta^\Theta_\cS = \Theta \Theta^\top \beta^\star$, we have 
	\begin{align*}
		\beta_\cT - \beta^\Theta_\cS & = (I_d - \Theta \Theta^\top) \beta^\star + \Delta (\tau^2 I_r + \Lambda_\cT)^{-1} \Lambda_\cT \gamma = \Delta \left[\Delta^\top \beta^\star + (\tau^2 I_r + \Lambda_\cT)^{-1} \Lambda_\cT \gamma\right] \;, \\
		\beta_\cT - \beta_\cS & = \Delta \left[(\tau^2 I_r + \Lambda_\cT)^{-1} \Lambda_\cT \gamma - (\tau^2 I_r + \Lambda_\cS)^{-1} \Lambda_\cS \gamma \right] \;,
	\end{align*}
	where the first equation uses $I_d = \Theta \Theta^\top + \Delta \Delta^\top$. From \eqref{eq:second_moment}, we can derive the following and thus complete the proof,
	\begin{align*}
		\|\beta_\cT - \beta^\Theta_\cS\|_{\E_\cT[X X^\top]}^2 & = \|\Delta^\top \beta^\star + (\tau^2 I_r + \Lambda_\cT)^{-1} \Lambda_\cT \gamma\|_{\tau^2 I_r  + \Lambda_\cT}^2 \;, \\
		\|\beta_\cT - \beta_\cS\|_{\E_\cT[X X^\top]}^2 & = \|(\tau^2 I_r + \Lambda_\cS)^{-1} \Lambda_\cS \gamma - (\tau^2 I_r + \Lambda_\cT)^{-1}\Lambda_\cT \gamma\|_{\tau^2 I_r  + \Lambda_\cT}^2 \;.
	\end{align*}
\end{proof}

\subsection{Proofs in Section~\ref{sec:methodology}}

\begin{proof}[Proof of Proposition \ref{prop:equality}]
	By definition of $R_\cE$, one can deduce that 
	\begin{equation}
		\label{eq:R_difference}
		R_\cT(\beta) - R_\cS(\beta) = \langle \beta, (\Sigma_\cT - \Sigma_\cS) \beta \rangle - 2\langle \E_\cT[X Y] - \E_\cS[X Y], \beta \rangle + \E_\cT[Y^2] - \E_\cS[Y^2] \;.
	\end{equation}
	From \eqref{eq:E[XY]}, we have
	\begin{equation}
		\label{eq:E[XY]_difference}
		\E_\cT[X Y] - \E_\cS[X Y] = (\Sigma_\cT - \Sigma_\cS) \beta^\star + \Delta (\Lambda_\cT -\Lambda_\cS) \gamma = (\Sigma_\cT - \Sigma_\cS) (\beta^\star + \Delta \gamma) \;,
	\end{equation}
	where the second equality uses $\Delta^\top \Delta = I_r$ from Assumption \ref{asmp:base} and 
	\begin{equation}
		\label{eq:second_moment_difference}
		\Sigma_\cT - \Sigma_\cS = \Delta (\Lambda_\cT -\Lambda_\cS) \Delta^\top
	\end{equation}
	from \eqref{eq:second_moment_based}. Meanwhile, from \eqref{eq:model_Y} and $\E[U] = 0$, we have 
	\begin{equation*}
		\E_\cE[Y^2] = \langle \beta^\star, \Sigma_\cE \beta^\ast \rangle + \langle \gamma, \Lambda_\cE \gamma \rangle + 2 \langle \beta^\star, \E_\cE[X E^\top] \gamma \rangle + \E[U^2] \;.
	\end{equation*}
	Therefore, 
	\begin{equation}
		\label{eq:E[Y^2]_difference}
		\begin{split}
			\E_\cT[Y^2] - \E_\cS[Y^2] 
			& = \langle \beta^\star, (\Sigma_\cT - \Sigma_\cS) \beta^\star \rangle + \langle \gamma, (\Lambda_\cT -\Lambda_\cS) \gamma \rangle + 2 \langle \beta^\star, (\E_\cT[X E^\top]-\E_\cS[X E^\top]) \gamma \rangle \\
			& = \langle \beta^\star, (\Sigma_\cT - \Sigma_\cS) \beta^\star \rangle + \langle \gamma, (\Lambda_\cT -\Lambda_\cS) \gamma \rangle + 2 \langle \beta^\star, \Delta (\Lambda_\cT -\Lambda_\cS) \gamma \rangle \\
			& = \langle \beta^\star, (\Sigma_\cT - \Sigma_\cS) \beta^\star \rangle + \langle \Delta \gamma, (\Sigma_\cT - \Sigma_\cS) \Delta \gamma \rangle + 2 \langle \beta^\star, (\Sigma_\cT - \Sigma_\cS) \Delta \gamma \rangle \;,			
		\end{split}	
	\end{equation}
	where the second equality is due to \eqref{eq:model_X} and Assumption \ref{asmp:base} and the third equality uses $\Delta^\top \Delta = I_r$ and \eqref{eq:second_moment_difference} again. Combining \eqref{eq:R_difference}, \eqref{eq:E[XY]_difference}, and \eqref{eq:E[Y^2]_difference}, we have \eqref{eq:R_difference_equality}.
\end{proof}
\begin{proof}[Proof of Proposition \ref{prop:unifUB}]
	By Proposition \ref{prop:equality} and Young's inequality, $\forall \xi >0$,
	\begin{align*}
		0 \leq R_{\cT}(\beta) - R_\cS(\beta) \leq (1+\xi) \langle \beta , D \beta \rangle + (1+\xi^{-1}) \langle \beta^\star + \Delta \gamma , D (\beta^\star + \Delta \gamma) \rangle \;, 
	\end{align*}
	where the second term on the right-hand side does not depend on $\beta$.
	To bound the first term, plug in $\beta = V \alpha$
	\begin{align*}
		\langle \beta , D \beta \rangle = \langle \alpha \alpha^\top , V^\top D V \rangle \leq \|\alpha\|^2 \| V^\top D V \|_{\mathrm{F}} \;.
	\end{align*}
    Apply Young's inequality again, $\forall \zeta > 0$,
    \begin{align*}
        \|\alpha\|^2 \| V^\top D V \|_{\mathrm{F}} \leq \frac{\zeta}{2}\|\alpha\|^4 + \frac{1}{2\zeta}\|V^\top D V\|_{\mathrm{F}}^2 \;.
    \end{align*}
    The proof follows.
\end{proof}

\begin{proof}[Proof of Proposition \ref{prop:riem-grad}]
   Let $\xi_V = \mathrm{grad} ~\Phi_{\upsilon, \eta}(V)$. As the Riemannian gradient is equal to the projection of the usual gradient, we have
	\begin{equation*}
		\mathrm{grad} ~\bar\Phi_{\upsilon, \eta}(V) = P_V(\mathrm{grad} ~\Phi_{\upsilon, \eta}(V)) = P_V(\xi_V) \;,		
	\end{equation*}
	where 
	\begin{align}
		P_{V}(\xi) = (I_d - V V^\top)\xi + V \mathrm{skew} (V^\top \xi) \in T_V \mathrm{St}(d,\ell) \;.
	\end{align}
	We calculate $\xi_V$ with an application of the Envelope Theorem, see, for example, \citet{milgrom2002envelope}. Indeed, the inner minimization, $\min_{\alpha \in \mathbb{R}^{\ell}} F_{\upsilon, \eta}(V,\alpha)
    $, admits a unique minimizer $\alpha_{V} := (V^\top \Sigma_\cS V + \upsilon I_{\ell} )^{-1} V^\top \E_{\cS}[XY]$. And so,
    \begin{align*}
        \xi_V &= \mathrm{grad}~\Phi_{\upsilon, \eta}(V) = \mathrm{grad}~ F_{\upsilon, \eta}(V,\alpha) \big|_{\alpha = \alpha_V} \\
        & = \left( \Sigma_\cS V \alpha_{V} - \E_{\cS}[XY] \right)\alpha_{V}^\top + \eta DVV^\top D V \;.
    \end{align*}
	We evaluate $\mathrm{skew} (V^\top \xi)$. Note that
	\begin{align*}
		V^\top \xi &= \left( V^\top \Sigma_\cS V \alpha_{V} - V^\top \E_{\cS}[XY] \right)\alpha_{V}^\top + \eta V^\top DVV^\top D V \\
		&= - \upsilon \alpha_V\alpha_{V}^\top + \eta V^\top DVV^\top D V
	\end{align*}
	is a symmetric matrix, thus $\mathrm{skew} (V^\top \xi) = 0$.
    Therefore:
    \begin{align*}
        \mathrm{grad} ~\bar \Phi_{\upsilon, \eta}(V) = P_V(\xi_V) = (I_d - V V^\top) \left((\Sigma_\cS V \alpha_V - \E_{\cS}[XY]) \alpha_V^\top  + \eta DVV^\top D V \right) 
    \end{align*}
	as required.
\end{proof}

\subsection{Proofs in Section~\ref{sec:theory}}
\begin{lemma}
	\label{lem:eigenvalue-bounds-loewner}
	For any positive semi-definite matrix $\Sigma \in \mathbb{R}^{d \times d}$ and $V \in \mathrm{St}(d, \ell)$, any $\upsilon >0$, the following inequality holds in Loewner ordering,
	\begin{align*}
		\tfrac{\upsilon}{\upsilon+\lambda_{\max}(\Sigma)} \cdot I_d \preceq I_d -   \Sigma^{1/2} V (V^\top \Sigma V + \upsilon I_{\ell} )^{-1} V^\top \Sigma^{1/2} \preceq I_d \;, \\ 
		 \Sigma^{1/2} V (V^\top \Sigma V + \upsilon I_{\ell} )^{-2} V^\top \Sigma^{1/2} \preceq \tfrac{1}{4\upsilon} \cdot I_d \;. 
	\end{align*}
\end{lemma}
\begin{proof}[Proof of Lemma~\ref{lem:eigenvalue-bounds-loewner}]
	The proof uses singular value decomposition (SVD) of $\Sigma^{1/2} V$. Let $\Sigma^{1/2} V = A D B^\top, \ A \in \R^{d \times d}, B \in \R^{\ell \times \ell}, \ D \in \R^{d \times \ell}$, $A^\top A = I_d, \ B B^\top = I_\ell$ and $D$ is rectangular diagonal. We tackle the first inequality. Note 
	\begin{align*}
		I_d - \Sigma^{1/2} V (V^\top \Sigma V + \upsilon I_{\ell} )^{-1} V^\top \Sigma^{1/2} &= I_d - A D B^\top (B D^\top D B^\top + \upsilon I_{\ell})^{-1} B D^\top A^\top \\
		&= A(I_d - D (D^\top D + \upsilon I_{\ell})^{-1} D^\top) A^\top \;,
	\end{align*}
	where the second equality uses the Woodbury matrix identity. Here,
	$I_d - D (D^\top D + \upsilon I_{\ell})^{-1} D^\top$ is a diagonal matrix with entries 
	\begin{equation*}
		(I_d - D (D^\top D + \upsilon I_{\ell})^{-1} D^\top)_{ii} =
		\begin{cases}
			\tfrac{\upsilon}{\upsilon + D_{ii}^2} & \text{for} ~ i \leq \ell \;, \\
			1 & \text{otherwise} \;.
		\end{cases}
	\end{equation*}
	Note that $D_{ii}^2 = \lambda_i(V^\top \Sigma V)$ where $\lambda_i$ denotes the $i$-th largest eigenvalue. We can verify $ \lambda_i(V^\top \Sigma V) \leq  \lambda_{\mathrm{max}}(\Sigma)$ as $V \in \mathrm{St}(d, \ell)$, thus we arrive at $\tfrac{\upsilon}{\upsilon+\lambda_{\max}(\Sigma)} \leq \tfrac{\upsilon}{\upsilon + D_{ii}^2} \leq 1$. The second inequality is established in a similar fashion. Reusing the SVD, we have
	\begin{align*}
		\Sigma^{1/2} V (V^\top \Sigma V + \upsilon I_{\ell} )^{-2} V^\top \Sigma^{1/2} = A D (D^\top D + \upsilon I_{\ell})^{-2} D^\top A^\top \;.
	\end{align*}
	The entries of the inner diagonal matrix read,
	\begin{equation*}
		(D (D^\top D + \upsilon I_{\ell})^{-2} D^\top)_{ii} =
		\begin{cases}
			\tfrac{D_{ii}^2 }{(D_{ii}^2 + \upsilon)^2}  & \text{for} ~ i \leq \ell \;, \\
			0 & \text{otherwise} \;.
		\end{cases} 
	\end{equation*}
	Then, $\tfrac{D_{ii}^2 }{(D_{ii}^2 + \upsilon)^2} \leq \max_{d > 0}  \tfrac{d^2 }{(d^2 + \upsilon)^2} = \tfrac{1}{4\upsilon}$.
\end{proof}

\begin{proof}[Proof of Theorem~\ref{thm:alignment-to-invariant-subspace}]
	Let $M := \Lambda_{\cT} - \Lambda_{\cS}$. Under Assumption~\ref{asmp:richer-target}, $M$ is a positive definite matrix. 
	The first-order condition of the Stiefel manifold optimization \eqref{eqn:opt-stiefel} reads
	\begin{align}
		\label{eqn:stationarity}
		(I_d - VV^\top) (  \E_{\cS}[XY] - \E_{\cS}[ XX^\top] V \alpha_V) \alpha_V^\top = \eta  (I_d - VV^\top) D V V^\top D V \;.
	\end{align}
	Define the canonical angles $A := V^\top \Delta \in \R^{\ell \times r}$. Taking the Frobenius norm of the right-hand side of \eqref{eqn:stationarity}, we have
	\begin{align*}
		& \eta^2 \mathrm{Tr}\big[ V^\top D V V^\top D (I_d - V V^\top) D V V^\top D V  \big] \\
		&= \eta^2 \mathrm{Tr}\big[ A M A^\top A M (I_{r} - A^\top A) M A^\top A M A^\top \big] \\
		&\geq \eta^2 \lambda_{\min} (I_{r} - A^\top A) \cdot \lambda_{\min}(M) \cdot \mathrm{Tr}\big[ A M A^\top A M A^\top A M A^\top \big] \\
		&\geq \eta^2 \lambda_{\min} (I_{r} - A^\top A) \cdot \lambda_{\min}(M)^4 \cdot \| A A^\top \|_{\mathrm{op}}^3 \;.
	\end{align*}
	In the inequalities above, we use the fact that if $B \succeq C$, then $D B D^\top \succeq D C D^\top$, in Loewner ordering. The last inequality leverages the additional facts that $\mathrm{Tr}\big[ (A M A^\top )^3 \big] \geq \| A M A^\top \|^3_{\mathrm{op}}$ and $\| A M A^\top \|_{\mathrm{op}} \geq \lambda_{\min}(M) \| A A^\top \|_{\mathrm{op}}$.
	
	The Frobenius norm on the left-hand side of \eqref{eqn:stationarity} is
	\begin{align}
		\label{eqn:LHS-terms}
		\| (I_d - VV^\top) (  \E_{S}[XY] - \E_{S}[ XX^\top] V \alpha_V) \|^2 \cdot  \| \alpha_V \|^2 \;.
	\end{align}
	Recall the two facts in Lemma~\ref{lem:eigenvalue-bounds-loewner},
	\begin{align}
		\label{eqn:fact-eigenvalue-1}
		\tfrac{\upsilon}{\upsilon+\lambda_{\max}(\Sigma)} \cdot I_d \preceq I_d -   \Sigma^{1/2} V (V^\top \Sigma V + \upsilon I_{\ell} )^{-1} V^\top \Sigma^{1/2} \preceq I_d \;, \\ 
		\label{eqn:fact-eigenvalue-2} 
		 \Sigma^{1/2} V (V^\top \Sigma V + \upsilon I_{\ell} )^{-2} V^\top \Sigma^{1/2} \preceq \tfrac{1}{4\upsilon} \cdot I_d  \;. 
	\end{align}
	Now we control each term in \eqref{eqn:LHS-terms}.
	
	For the first term in \eqref{eqn:LHS-terms}, we have
	\begin{align*}
		& (I_d - VV^\top) (  \E_{S}[XY] - \E_{S}[ XX^\top] V \alpha_V) \\
		& = ( I_d - VV^\top) \big( I_d -  \E_{\cS}[X X^\top]  V (V^\top \E_{\cS}[X X^\top] V + \upsilon I_{\ell} )^{-1} V^\top \big) \E_{\cS}[XY] \;.
	\end{align*}
	Therefore, denote $\Sigma := \E_{\cS}[X X^\top] $
	\begin{align*}
		&\| (I_d - VV^\top) (  \E_{S}[XY] - \E_{S}[ XX^\top] V \alpha_V) \|^2 \\
		&\leq \| \big( I_d -  \E_{\cS}[X X^\top]  V (V^\top \E_{\cS}[X X^\top] V + \upsilon I_{\ell} )^{-1} V^\top \big) \E_{\cS}[XY]  \|^2 \\
		& = \| \Sigma^{1/2} \big( I_d -   \Sigma^{1/2} V (V^\top \Sigma V + \upsilon I_{\ell} )^{-1} V^\top \Sigma^{1/2} \big) \Sigma^{-1/2} \E_{\cS}[XY]  \|^2 \\
		& \leq \lambda_{\max} (\Sigma) \cdot \| \big( I_d -   \Sigma^{1/2} V (V^\top \Sigma V + \upsilon I_{\ell} )^{-1} V^\top \Sigma^{1/2} \big) \Sigma^{-1/2} \E_{\cS}[XY]  \|^2 \quad \text{by \eqref{eqn:fact-eigenvalue-1}} \\
		& \leq \lambda_{\max} (\Sigma) \cdot \| \Sigma^{-1/2} \E_{\cS}[XY]  \|^2 \;.
	\end{align*}
	
	For the second term in \eqref{eqn:LHS-terms}, we know
	\begin{align*}
		 \| \alpha_V \|^2 &= \| (V^\top \Sigma V + \upsilon I_{\ell} )^{-1} V^\top  \E_{S}[XY] \|^2 \\
		 &\leq \lambda_{\max}\big(  \Sigma^{1/2} V (V^\top \Sigma V + \upsilon I_{\ell} )^{-2} V^\top \Sigma^{1/2} \big) \cdot \| \Sigma^{-1/2} \E_{\cS}[XY]  \|^2  \quad \text{by \eqref{eqn:fact-eigenvalue-2}}\\
		 &\leq \frac{1}{4\upsilon }  \| \Sigma^{-1/2} \E_{\cS}[XY]  \|^2 \;.
	\end{align*}
	
	Putting things together, for $V \in \cS_{\Delta}(\delta)$, we have that the canonical angles $A = V^\top \Delta \in \R^{\ell \times r}$ satisfies $1- \lambda_{\max}(A^\top A) \geq 1- \| A \|^2_{\mathrm{op}} \geq \delta$,  and that
	\begin{align*}
		\eta^2 \delta \lambda_{\min}(M)^4 \cdot \| A^\top A  \|_{\mathrm{op}}^3 \leq   \lambda_{\max} (\Sigma_{\cS}) \cdot \| \Sigma_{\cS}^{-1/2} \E_{\cS}[XY]  \|^2 \frac{1}{4\upsilon}  \| \Sigma_{\cS}^{-1/2} \E_{\cS}[XY]  \|^2 \;.
	\end{align*}
	Therefore, we have proved
	\begin{align*}
		\| V^\top \Delta \|_{\mathrm{op}}^6 \leq \frac{\delta^{-1}  \lambda_{\max}(\Sigma_{\cS})   \| \Sigma_{\cS}^{-1/2} \E_{\cS}[XY]  \|^4  }{\lambda_{\min}(\Delta^\top D \Delta)^4} \frac{1}{4\upsilon\eta^2} \;.
	\end{align*}
\end{proof}

\begin{proof}[Proof of Theorem~\ref{thm:stability-source-target}]
	In view of \eqref{eq:R_difference_equality}, by Young's inequality, we have
	\begin{align*}
		 R_\cT(\beta^{\upsilon, \eta}) - R_\cS(\beta^{\upsilon, \eta}) \leq  (1+\epsilon) \langle \beta^\star + \Delta \gamma , D (\beta^\star + \Delta \gamma) \rangle + (1+\epsilon^{-1}) \langle V \alpha_V, D V \alpha_V \rangle \;.
	\end{align*}
	Note 
	\begin{align*}
		\langle V \alpha_V, D V \alpha_V \rangle &\leq \| \alpha_V \|^2  \cdot \| V^\top D V \|_{\mathrm{op}} \\
		& \leq \| \alpha_V \|^2  \cdot \| V^\top \Delta M \Delta^\top V \|_{\mathrm{op}} \\
		& \leq \| \alpha_V \|^2  \cdot \lambda_{\max}(M) \| V^\top \Delta \|_{\mathrm{op}}^2 \;.
	\end{align*}
	Recall the fact that
	\begin{align*}
		\| \alpha_V \|^2 \leq \frac{1}{4\upsilon }  \| \Sigma^{-1/2} \E_{\cS}[XY]  \|^2
	\end{align*}
	and the bound on $\| V^\top \Delta \|_{\mathrm{op}}^6$ established in Theorem~\ref{thm:alignment-to-invariant-subspace}, we finish the proof after separating out the terms independent of $\eta, \upsilon$ and defining them as $\mathsf{S}_{\epsilon, \delta}$.
\end{proof}

\begin{proof}[Proof of Theorem~\ref{thm:finite-sample-error}]
	First, we claim
	\begin{equation}
		\label{eq:error-decomposition}
		\Phi_{\upsilon, \eta}(\widehat{V}) - \min_{V \in \mathrm{St}(d, \ell)} \Phi_{\upsilon, \eta}(V) \le 2 \sup_{V \in \mathrm{St}(d, \ell)} |\widehat{\Phi}_{\upsilon, \eta}(V) - \Phi_{\upsilon, \eta}(V)| \;.
	\end{equation}
	To see this, pick a minimizer $V^\star$ of $\Phi_{\upsilon, \eta}$ over $\mathrm{St}(d, \ell)$. Then, we have
	\begin{equation*}
		\begin{split}
			\Phi_{\upsilon, \eta}(\widehat{V}) - \Phi_{\upsilon, \eta}(V^\star)
			& = \Phi_{\upsilon, \eta}(\widehat{V}) - \widehat{\Phi}_{\upsilon, \eta}(\widehat{V}) + \widehat{\Phi}_{\upsilon, \eta}(\widehat{V}) - \widehat{\Phi}_{\upsilon, \eta}(V^\star) + \widehat{\Phi}_{\upsilon, \eta}(V^\star) - \Phi_{\upsilon, \eta}(V^\star) \\
			& \le \Phi_{\upsilon, \eta}(\widehat{V}) - \widehat{\Phi}_{\upsilon, \eta}(\widehat{V}) + \widehat{\Phi}_{\upsilon, \eta}(V^\star) - \Phi_{\upsilon, \eta}(V^\star) \\
			& \le 2 \sup_{V \in \mathrm{St}(d, \ell)} |\widehat{\Phi}_{\upsilon, \eta}(V) - \Phi_{\upsilon, \eta}(V)| \;,
		\end{split}
	\end{equation*}
	where the first inequality follows as $\widehat{V}$ minimizes $\widehat{\Phi}_{\upsilon, \eta}$.

	Next, let $\widehat{R}_\cS$, $\widehat{\Sigma}_\cT$, and $\widehat{\Sigma}_\cS$ be the empirical versions of $R_\cS$, $\Sigma_\cT$, and $\Sigma_\cS$, respectively. Then, we have
	\begin{equation}
		\label{eq:sup-error-decomposition}
		\begin{split}
			2 (\widehat{\Phi}_{\upsilon, \eta}(V) - \Phi_{\upsilon, \eta}(V))
			& = \min_{\alpha \in \R^\ell} \bigg(\widehat{R}_\cS(V \alpha) + \upsilon \|\alpha\|^2\bigg) - \min_{\alpha \in \R^\ell} \bigg(R_\cS(V \alpha) + \upsilon \|\alpha\|^2\bigg) \\
			& \qquad + \frac{\eta}{2} \left(\|V^\top (\widehat{\Sigma}_\cT - \widehat{\Sigma}_\cS) V\|_{\mathrm{F}}^2 - \|V^\top (\Sigma_\cT - \Sigma_\cS) V\|_{\mathrm{F}}^2\right) \;.        
		\end{split}
	\end{equation}
	For any $V \in \mathrm{St}(d, \ell)$, we have $\alpha_V = (V^\top \Sigma_\cS V + \upsilon I_\ell)^{-1} V^\top \E_\cS[X Y]$, which yields
	\begin{equation*}
		\|\alpha_V\| \le \|(V^\top \Sigma_\cS V + \upsilon I_\ell)^{-1}\|_{\mathrm{op}} \cdot \|V^\top \E_\cS[X Y]\| \le \frac{M^2}{\|V^\top \Sigma_\cS V + \upsilon I_\ell\|_{\mathrm{op}}} \le \frac{M^2}{\upsilon} \;.
	\end{equation*}
	Similarly, $\|\widehat{\alpha}_V\| \le \frac{M^2}{\upsilon}$ for $\widehat{\alpha}_V := \argmin_{\alpha \in \R^\ell} \widehat{F}_{\upsilon, \eta}(V, \alpha) = (V^\top \widehat{\Sigma}_\cS V + \upsilon I_\ell)^{-1} V^\top \widehat{\E}_\cS[X Y]$, where $\widehat{F}_{\upsilon, \eta}, \widehat{\E}_\cS$ denote the empirical versions of $F_{\upsilon, \eta}, \E_\cS$, respectively. Therefore, 
	\begin{equation*}
		\begin{split}
			& \left|\min_{\alpha \in \R^\ell} \bigg(\widehat{R}_\cS(V \alpha) + \upsilon \|\alpha\|^2\bigg) - \min_{\alpha \in \R^\ell} \bigg(R_\cS(V \alpha) + \upsilon \|\alpha\|^2\bigg)\right| \\
			& \quad = \left|\min_{\substack{\alpha \in \R^\ell \\ \|\alpha\| \le M^2 / \upsilon}} \bigg(\widehat{R}_\cS(V \alpha) + \upsilon \|\alpha\|^2\bigg) - \min_{\substack{\alpha \in \R^\ell \\ \|\alpha\| \le M^2 / \upsilon}} \bigg(R_\cS(V \alpha) + \upsilon \|\alpha\|^2\bigg)\right| \\
			& \quad \le \sup_{\substack{\alpha \in \R^\ell \\ \|\alpha\| \le M^2 / \upsilon}} \left|\widehat{R}_\cS(V \alpha) - R_\cS(V \alpha)\right| \;.
		\end{split}
	\end{equation*}
	Combining \eqref{eq:error-decomposition}, \eqref{eq:sup-error-decomposition}, and
	\begin{equation*}
		\sup_{V \in \mathrm{St}(d, \ell)} \sup_{\substack{\alpha \in \R^\ell \\ \|\alpha\| \le M^2 / \upsilon}} \left|\widehat{R}_\cS(V \alpha) - R_\cS(V \alpha)\right| = \sup_{\substack{\beta \in \R^d \\ \|\beta\| \le M^2 / \upsilon}} \left|\widehat{R}_\cS(\beta) - R_\cS(\beta)\right| \;,
	\end{equation*}
	we have
	\begin{equation}
		\label{eq:final-error-decomposition}
		\begin{split}
			& \Phi_{\upsilon, \eta}(\widehat{V}) - \min_{V \in \mathrm{St}(d, \ell)} \Phi_{\upsilon, \eta}(V) \\
			& \quad = \sup_{\substack{\beta \in \R^d \\ \|\beta\| \le M^2 / \upsilon}} \left|\widehat{R}_\cS(\beta) - R_\cS(\beta)\right| \\
			& \quad \qquad + \frac{\eta}{2} \sup_{V \in \mathrm{St}(d, \ell)} \left|\|V^\top (\widehat{\Sigma}_\cT - \widehat{\Sigma}_\cS) V\|_{\mathrm{F}}^2 - \|V^\top (\Sigma_\cT - \Sigma_\cS) V\|_{\mathrm{F}}^2\right| \;.
		\end{split}
	\end{equation}

	By the symmetrization argument, we have
	\begin{equation*}
		\E \sup_{\substack{\beta \in \R^d \\ \|\beta\| \le M^2 / \upsilon}} \left|\widehat{R}_\cS(\beta) - R_\cS(\beta)\right| \le \frac{2}{n} \E \sup_{\substack{\beta \in \R^d \\ \|\beta\| \le M^2 / \upsilon}} \left|\sum_{i = 1}^{n} \sigma_i (y_i - \langle \beta, x_i \rangle)^2\right| \;,
	\end{equation*}
	where $\sigma_1, \ldots, \sigma_n$ are i.i.d.\ Rademacher variables that are independent of $\{(x_i, y_i)\}_{i = 1}^{n}$ and $\{\tilde{x}_i\}_{i = 1}^{m}$. As $|y_i - \langle \beta, x_i \rangle| \le M + \frac{M^3}{\upsilon} =: L$ for all $i$ and $\|\beta\| \le M^2 / \upsilon$, we can apply the contraction principle of the Rademacher complexity to the $(2 L)$-Lipschitz function $\phi(t) = t^2$ on $[-L, L]$, which yields
	\begin{equation*}
		\begin{split}
			\E \sup_{\substack{\beta \in \R^d \\ \|\beta\| \le M^2 / \upsilon}} \left|\widehat{R}_\cS(\beta) - R_\cS(\beta)\right| 
			& \le \frac{8 (M + \frac{M^3}{\upsilon})}{n} \times \E \sup_{\substack{\beta \in \R^d \\ \|\beta\| \le M^2 / \upsilon}} \left|\sum_{i = 1}^{n} \sigma_i (y_i - \langle \beta, x_i \rangle)\right| \\
			& \le \frac{8 (M + \frac{M^3}{\upsilon})}{n} \left(\E \left|\sum_{i = 1}^{n} \sigma_i y_i\right| + \E \sup_{\substack{\beta \in \R^d \\ \|\beta\| \le M^2 / \upsilon}} \left|\sum_{i = 1}^{n} \sigma_i \langle \beta, x_i \rangle\right|\right) \;,        
		\end{split}
	\end{equation*}
	where the second inequality follows from the triangle inequality. We have 
	\begin{equation*}
		\E \left|\sum_{i = 1}^{n} \sigma_i y_i\right| \le \sqrt{\E \left|\sum_{i = 1}^{n} \sigma_i y_i\right|^2} = \sqrt{\E \sum_{i = 1}^{n} |Y_i|^2} \le M \sqrt{n} \;,
	\end{equation*}
	where the first inequality is Jensen's inequality, and the equality follows from the independence of $y_i$'s and $\sigma_i$'s. Similarly, we can deduce $\E \left\|\sum_{i = 1}^{n} \sigma_i x_i\right\| \le M \sqrt{n}$, which yields
	\begin{equation*}
		\E \sup_{\substack{\beta \in \R^d \\ \|\beta\| \le M^2 / \upsilon}} \left|\sum_{i = 1}^{n} \sigma_i \langle \beta, x_i \rangle\right|
		= \frac{M^2}{\upsilon} \E \left\|\sum_{i = 1}^{n} \sigma_i x_i\right\|
		\le \frac{M^3}{\upsilon} \sqrt{n} \;.
	\end{equation*}
	Therefore, 
	\begin{equation*}
		\E \sup_{\substack{\beta \in \R^d \\ \|\beta\| \le M^2 / \upsilon}} \left|\widehat{R}_\cS(\beta) - R_\cS(\beta)\right| \le \frac{8 (M + \frac{M^3}{\upsilon})^2}{\sqrt{n}} \;.
	\end{equation*}
	If we change a data point $(x_i, y_i)$ to $(x_i', y_i')$, then the change in $\widehat{R}_\cS(\beta)$ for any $\beta \in \R^d$ with $\|\beta\| \le M^2 / \upsilon$ can be bounded by 
	\begin{equation*}
		\frac{1}{n} |(y_i - \langle \beta, x_i \rangle)^2 - (y_i' - \langle \beta, x_i' \rangle)^2| \le \frac{1}{n} \sup_{\substack{(x, y) \in \R^d \times [-M, M] \\ \|x\| \le M}} (y - \langle \beta, x \rangle)^2 \le \frac{(M + \frac{M^3}{\upsilon})^2}{n} \;.
	\end{equation*}
	By McDiarmid's inequality, for any $\delta \in (0, 1)$, the inequality
	\begin{equation*}
		\sup_{\substack{\beta \in \R^d \\ \|\beta\| \le M^2 / \upsilon}} \left|\widehat{R}_\cS(\beta) - R_\cS(\beta)\right| \le \E \sup_{\substack{\beta \in \R^d \\ \|\beta\| \le M^2 / \upsilon}} \left|\widehat{R}_\cS(\beta) - R_\cS(\beta)\right| + (M + \frac{M^3}{\upsilon})^2 \sqrt{\frac{\log \frac{1}{\delta}}{2 n}}
	\end{equation*}
	holds with probability at least $1 - \delta$. Hence, 
	\begin{equation}
		\label{eq:risk_bound}
		\sup_{\substack{\beta \in \R^d \\ \|\beta\| \le M^2 / \upsilon}} \left|\widehat{R}_\cS(\beta) - R_\cS(\beta)\right| \le \frac{8 (M + \frac{M^3}{\upsilon})^2}{\sqrt{n}} + (M + \frac{M^3}{\upsilon})^2 \sqrt{\frac{\log \frac{1}{\delta}}{2 n}}
	\end{equation}
	holds with probability at least $1 - \delta$. 

	Next, for $V \in \mathrm{St}(d, \ell)$, we have
	\begin{equation*}
		\|V^\top (\Sigma_\cT - \Sigma_\cS) V\|_{\mathrm{F}} \le \|V^\top\|_{\mathrm{op}} \times \|(\Sigma_\cT - \Sigma_\cS)\|_{\mathrm{op}} \times \|V\|_{\mathrm{F}} = \sqrt{\ell} \cdot \|\Sigma_\cT - \Sigma_\cS\|_{\mathrm{op}} \;.
	\end{equation*}
	Similarly, 
	\begin{equation*}
		\|V^\top (\widehat{\Sigma}_\cT - \widehat{\Sigma}_\cS) V\|_{\mathrm{F}} \le \sqrt{\ell} (\|\widehat{\Sigma}_\cT\|_{\mathrm{op}} + \|\widehat{\Sigma}_\cS\|_{\mathrm{op}}) \le \sqrt{\ell} \left(\frac{1}{n} \sum_{i = 1}^{n} \|x_i\|^2 + \frac{1}{m} \sum_{i = 1}^{m} \|\tilde{x}_i\|^2\right) \le 2 \sqrt{\ell} M^2 \;.
	\end{equation*}
	Therefore, 
	\begin{equation*}
		\begin{split}
			& \sup_{V \in \mathrm{St}(d, \ell)} \left|\|V^\top (\widehat{\Sigma}_\cT - \widehat{\Sigma}_\cS) V\|_{\mathrm{F}}^2 - \|V^\top (\Sigma_\cT - \Sigma_\cS) V\|_{\mathrm{F}}^2\right| \\
			& \quad \le \sqrt{\ell} \left(\|\Sigma_\cT - \Sigma_\cS\|_{\mathrm{op}} + 2 M^2\right) \times \sup_{V \in \mathrm{St}(d, \ell)} \left|\|V^\top (\widehat{\Sigma}_\cT - \widehat{\Sigma}_\cS) V\|_{\mathrm{F}} - \|V^\top (\Sigma_\cT - \Sigma_\cS) V\|_{\mathrm{F}}\right|         
		\end{split}
	\end{equation*}
	Meanwhile, for $V \in \mathrm{St}(d, \ell)$, we have
	\begin{equation*}
		\begin{split}
			\left|\|V^\top (\widehat{\Sigma}_\cT - \widehat{\Sigma}_\cS) V\|_{\mathrm{F}} - \|V^\top (\Sigma_\cT - \Sigma_\cS) V\|_{\mathrm{F}}\right| 
			& \le \|V^\top (\widehat{\Sigma}_\cT - \Sigma_\cT) V\|_{\mathrm{F}} + \|V^\top (\widehat{\Sigma}_\cS - \Sigma_\cS) V\|_{\mathrm{F}} \\
			& \le \sqrt{\ell} \left(\|\widehat{\Sigma}_\cT - \Sigma_\cT\|_{\mathrm{op}} + \|\widehat{\Sigma}_\cS - \Sigma_\cS\|_{\mathrm{op}}\right) \;.        
		\end{split}
	\end{equation*}
	Therefore, we have
	\begin{equation*}
		\begin{split}
			& \sup_{V \in \mathrm{St}(d, \ell)} \left|\|V^\top (\widehat{\Sigma}_\cT - \widehat{\Sigma}_\cS) V\|_{\mathrm{F}}^2 - \|V^\top (\Sigma_\cT - \Sigma_\cS) V\|_{\mathrm{F}}^2\right| \\
			& \quad \le \ell \left(\|\Sigma_\cT - \Sigma_\cS\|_{\mathrm{op}} + 2 M^2\right) \times \left(\|\widehat{\Sigma}_\cT - \Sigma_\cT\|_{\mathrm{op}} + \|\widehat{\Sigma}_\cS - \Sigma_\cS\|_{\mathrm{op}}\right) \;.        
		\end{split}
	\end{equation*}
	We bound the terms $\|\widehat{\Sigma}_\cT - \Sigma_\cT\|_{\mathrm{op}}$ and $\|\widehat{\Sigma}_\cS - \Sigma_\cS\|_{\mathrm{op}}$ using the standard covariance estimation results based on the matrix Bernstein inequality. For instance, based on Theorem 1.6.2 and Section 1.6.3 of \citet{tropp2015introduction}, we have for any $t > 0$,
	\begin{equation*}
		\Pr(\|\widehat{\Sigma}_\cT - \Sigma_\cT\|_{\mathrm{op}} \ge t) \le 2 d \cdot \exp\left(-\frac{t^2 / 2}{\frac{M^2 \|\Sigma_\cT\|_{\mathrm{op}}}{n} + \frac{2 M^2 t}{3 n}}\right) \;.
	\end{equation*}
	From this, we can deduce that for any $\delta \in (0, 1)$, the inequality
	\begin{equation*}
		\|\widehat{\Sigma}_\cT - \Sigma_\cT\|_{\mathrm{op}} < \frac{4 M^2 \log \frac{2 d}{\delta}}{3 n} + \sqrt{\frac{2 M^2 \|\Sigma_\cT\|_{\mathrm{op}} \log \frac{2 d}{\delta}}{n}}
	\end{equation*}
	holds with probability at least $1 - \delta$. Similarly, we can bound $\|\widehat{\Sigma}_\cS - \Sigma_\cS\|_{\mathrm{op}}$. Combining these results with \eqref{eq:final-error-decomposition} and \eqref{eq:risk_bound} through the union bound, the inequality
	\begin{equation*}
		\begin{split}
			& \Phi_{\upsilon, \eta}(\widehat{V}) - \min_{V \in \mathrm{St}(d, \ell)} \Phi_{\upsilon, \eta}(V) \\
			& \quad \le \frac{8 (M + \frac{M^3}{\upsilon})^2}{\sqrt{n}} + (M + \frac{M^3}{\upsilon})^2 \sqrt{\frac{\log \frac{3}{\delta}}{2 n}} \\
			& \quad \qquad + \frac{\eta \ell (\|\Sigma_\cT - \Sigma_\cS\|_{\mathrm{op}} + 2 M^2)}{2} \times \left(\frac{8 M^2 \log \frac{6 d}{\delta}}{3 n} + \sqrt{\frac{8 M^2 \max(\|\Sigma_\cT\|_{\mathrm{op}}, \|\Sigma_\cS\|_{\mathrm{op}}) \log \frac{6 d}{\delta}}{n}}\right) \\
			& \quad \le 9 (M + \frac{M^3}{\upsilon})^2 \sqrt{\frac{\log \frac{3}{\delta}}{n}} + \sC_{M, \Sigma_\cT, \Sigma_\cS} \cdot \eta \ell \left(\frac{\log \frac{6 d}{\delta}}{n} + \sqrt{\frac{\log \frac{6 d}{\delta}}{n}}\right)
		\end{split}
	\end{equation*}
	holds with probability at least $1 - \delta$,
	where
	\begin{equation*}
		\sC_{M, \Sigma_\cT, \Sigma_\cS} = \frac{\|\Sigma_\cT - \Sigma_\cS\|_{\mathrm{op}} + 2 M^2}{2} \times \max\left(\frac{8 M^2}{3}, \sqrt{8 M^2 \max(\|\Sigma_\cT\|_{\mathrm{op}}, \|\Sigma_\cS\|_{\mathrm{op}})}\right) \;.
	\end{equation*}
\end{proof}

\end{document}